\documentclass[journal]{IEEEtran}

\usepackage[utf8]{inputenc}
\usepackage{amsmath,amssymb}
\usepackage{mathtools}
\usepackage{hyperref,cleveref}

\usepackage{cite}

\usepackage{caption}
\usepackage{color}
\usepackage{subcaption,epsfig}
\usepackage[outdir=./]{epstopdf}
\usepackage{pdfpages}
\usepackage{tabularx}
\usepackage{wasysym} 

\usepackage{comment} 

\usepackage{tikz, pgfplots,pgfplotstable}
\usetikzlibrary{calc,positioning,plotmarks,shapes,snakes, arrows, arrows.meta}
\usepgfplotslibrary{fillbetween}
\pgfplotsset{compat=newest}
\usepackage{varwidth}  

\usepackage{diagbox}

\usepackage{makecell}

\usepackage{booktabs}

\usepackage{xr}
\definecolor{TUDa-1b}{HTML}{005AA9}
\definecolor{TUDa-3b}{HTML}{009D81}
\definecolor{TUDa-7b}{HTML}{F5A300}
\definecolor{TUDa-9b}{HTML}{E6001A}
\definecolor{TUDa-11b}{HTML}{721085}
\newtheorem{theorem}{\textbf{Theorem}}

\newenvironment{proof}[1][$\!\!$]{{\it Proof  #1:  }}
{\hfill$\blacksquare$}

\tikzset{cross/.style={cross out, draw=black, minimum size=5pt, inner sep=0pt, outer sep=0pt},
cross/.default={1pt}}

\input{mySymbol.sty}
\usepackage{tikz}
\usepackage[framemethod=tikz]{mdframed}
\usepackage{graphicx}
\usepackage{algorithm,algorithmicx,algpseudocode}
\usepackage{multirow}
\usepackage{array} 
\usepackage{fontsize}
\usepackage{enumitem}
\usepackage{arydshln}

\usetikzlibrary{arrows,decorations.pathmorphing,backgrounds,positioning,fit,petri}
\usetikzlibrary{shadows,shapes.misc,positioning,calc}
\usetikzlibrary{decorations.pathreplacing}

\usepackage{setspace}



\usepackage[normalem]{ulem} 

\newcommand{\polish}[2]{\textcolor{blue}{#2}}

\begin{document}

\setlength{\textfloatsep}{10pt plus 1pt minus 2pt}

\title{\vspace{-0.2cm} Attentional Graph Neural Network Is All You Need for Robust Massive Network Localization}

\author{\vspace{-0cm}
Wenzhong~Yan,~\IEEEmembership{Student Member,~IEEE},~Feng~Yin,~\IEEEmembership{Senior Member,~IEEE},\\~Juntao~Wang,~Geert~Leus,~\IEEEmembership{Fellow,~IEEE},~Abdelhak~M.~Zoubir,~\IEEEmembership{Fellow,~IEEE},~and~Yang~Tian

\thanks{
This paper is an extension of our work \cite{yan2021graph}, presented in the proceedings of ICASSP 2021. (Corresponding author: Feng Yin.)}
\thanks{Wenzhong~Yan,~Feng~Yin,~and~Juntao~Wang are with the School of Science \& Engineering, The Chinese University of Hong Kong, Shenzhen, Shenzhen 518172, China (email: wenzhongyan@link.cuhk.edu.cn, yinfeng@cuhk.edu.cn, juntaowang@link.cuhk.edu.cn).}
\thanks{Geert~Leus is with the Faculty of Electrical Engineering, Mathematics and Computer Science, Delft University of Technology, 2826 CD Delft, The Netherlands (e-mail: g.t.l.leus@tudelf.nl).}
\thanks{Abdelhak~M.~Zoubir is with the Signal Processing Group, Technische Universität
Darmstadt, Darmstadt 64283, Germany (e-mail: zoubir@spg.tu-darmstadt.de). 
}
\thanks{Yang~Tian is with Huawei Technologies Co., LTD., Shanghai 201206, China (email: yang.tian@huawei.com).}
}

\maketitle

\begin{abstract}
In this paper, we design Graph Neural Networks (GNNs) with attention mechanisms to tackle an important yet challenging nonlinear regression problem: massive network localization. We first review our previous network localization method based on Graph Convolutional Network (GCN), which can exhibit state-of-the-art localization accuracy, even under severe Non-Line-of-Sight (NLOS) conditions, by carefully preselecting a constant threshold for determining adjacency. As an extension, we propose a specially designed Attentional GNN (AGNN) model to resolve the sensitive thresholding issue of the GCN-based method and enhance the underlying model capacity. The AGNN comprises an Adjacency Learning Module (ALM) and Multiple Graph Attention Layers (MGAL), employing distinct attention architectures to systematically address the demerits of the GCN-based method, rendering it more practical for real-world applications. Comprehensive analyses are conducted to explain the superior performance of these methods, including a theoretical analysis of the AGNN's dynamic attention property and computational complexity, along with a systematic discussion of their robust characteristic against NLOS measurements. Extensive experimental results demonstrate the effectiveness of the GCN-based and AGNN-based network localization methods. Notably, integrating attention mechanisms into the AGNN yields substantial improvements in localization accuracy, approaching the fundamental lower bound and showing approximately 37\% to 53\% reduction in localization error compared to the vanilla GCN-based method across various NLOS noise configurations. Both methods outperform all competing approaches by far in terms of localization accuracy, robustness, and computational time, especially for considerably large network sizes.
\end{abstract}


\IEEEpeerreviewmaketitle


\section{Introduction}

\IEEEPARstart{G}{raph} Neural Networks (GNNs) have recently achieved excellent results in a variety of graph-related learning tasks, such as node classification, link prediction, and graph classification \cite{kipf2016semi, hamilton2017inductive, velivckovic2017graph, xu2018powerful,isufi2021edgenets}. By leveraging graph topology to aggregate neighbor information and enhance node representations \cite{ortega2018graph, gama2018convolutional}, GNNs offer a unique advantage in processing complex, interconnected datasets. Despite their effectiveness, GNNs have predominantly focused on classification tasks with discrete labels. However, nonlinear regression problems, which represent a substantial portion of practical applications in signal processing, remain less investigated.

In this paper, we delve into a classic yet challenging nonlinear regression problem: massive network localization. Over recent decades, a variety of canonical methods have been developed to tackle this issue, ranging from maximum likelihood \cite{Patwari03, yin2013and, Simonetto_Leus_2014, yin2013toa} and least-squares based \cite{Wymeersch2009} estimation methods to multi-dimensional scaling \cite{Costa2006}, mathematical programming \cite{Biswas_Ye_SDP, Tseng2007, xiong2021cooperative}, and Bayesian message passing approaches \cite{Ihler2005,jin2020bayesian}. Despite their diversity, a major challenge these methods face is the impact of Non-Line-of-Sight (NLOS) propagation, which can incur severe performance degradation of the localization accuracy due to severely biased distance estimation.

The prevalent strategy to combat NLOS effects is to perform NLOS identification for each link, and either discard or suppress the NLOS measurements \cite{Win_2010_NLOS_UWB}. However, this method requires large-scale offline calibration and substantial manpower. Alternatively, the NLOS effects can be dealt with from an algorithmic aspect. Based on the assumption that NLOS noise follows a certain probability distribution, maximum likelihood estimation based methods handling NLOS effects were developed in \cite{Yin_ECM, Chen_Wang_So_Poor_2012}. Yet, such methods may suffer from performance deterioration due to model mismatch. Moreover, network localization has been formulated as a regularized optimization problem in which the NLOS-inducing sparsity of the ranging-bias parameters was exploited \cite{jin2021exploiting}. Unfortunately, all the above methods are computationally expensive for massive networks, which makes them less feasible for large-scale applications.

Most recently, we have introduced a novel Graph Convolutional Network (GCN)-based method for network localization in \cite{yan2021graph}, which significantly advances the field of massive network localization by offering excellent accuracy, robustness, and efficiency without requiring extensive offline calibration or NLOS identification. However, some challenges persist in the vanilla GCN-based method, particularly concerning its reliance on empirically predefined thresholds for establishing adjacency and its insufficient model expressiveness. 
		
{Recent studies have explored advanced GNN-based techniques to improve wireless localization.
For instance, \cite{han2024mobility} proposed a graph-learning approach that incorporates temporal and directional mobility patterns into a time-varying graph structure, complemented by a mobility regularization term to improve performance.  \cite{tang2023csi} developed a GNN-based localization method that utilizes Channel State Information (CSI) by modeling amplitude-phase relationships as graph inputs, enabling efficient position estimation via a GraphSAGE architecture \cite{hamilton2017inductive}. These works illustrate the potential of domain-adapted graph-based methods to address challenges in wireless localization.}

The advent of attention mechanisms has recently significantly impacted the modeling of structured data, including spatial and temporal data \cite{bahdanau14neural,vaswani2017attention,velivckovic2017graph}. 
These mechanisms, by design, enable models to focus on the most relevant parts of the input data, enhancing their ability to learn complex patterns from the data.  {For example, \cite{wu2022multi} proposed a federated graph learning framework integrating GNNs with self-attention mechanisms for personalized Wi-Fi localization, effectively capturing spatial and client-specific dependencies. These advancements highlight} {the benefits of combining attention mechanisms with GNNs to further enhance network localization performance under diverse conditions.}

Motivated by these factors, in this paper, we propose an Attentional GNN (AGNN) model, comprising two crucial modules: Adjacency Learning Module (ALM) and Multiple Graph Attention Layers (MGAL), aimed at effectively mitigating the limitations inherent in the GCN-based network localization. Besides, we systematically analyze the superb properties of the proposed model. 
The main contributions can be summarized as follows.
\begin{itemize}[itemsep=0pt, topsep=2pt]
    \item To the best of our knowledge, this is the first work to integrate GNNs and attention mechanisms for network localization. Together with our previous GCN-based method \cite{yan2021graph}, we hope to pave a new path to solve massive network localization in our futuristic connected world.

    \item To remove the reliance on a predefined constant threshold for determining adjacency in the vanilla GCN-based method, we introduce an ALM that uses a tailored attention architecture to autonomously learn edge-specific thresholds. This module can characterize more reasonable and flexible node neighborhoods, particularly effective in filtering out NLOS-affected node pairs.

    \item To enhance the model expressiveness, we introduce MGAL, which leverage attention mechanisms to learn aggregation weights based on the features of each target node and its neighbors, instead of using constant weights in the GCN-based method. Moreover, differences in the distributions of the learned attention scores for Line-of-Sight (LOS) and NLOS links evidently show that our adopted attention mechanism is able to identify NLOS links, addressing a long-standing challenge in the field.

    \item To support the significant performance enhancement of the AGNN model, we perform a comprehensive analysis with the highlights on the dynamic attentional properties of our specially designed attention architectures, and a thorough complexity analysis for all competing models.

    \item To showcase the efficacy of the AGNN model, we conduct extensive experiments which revealed that AGNN most closely approaches the Cramér-Rao Bound (CRB) and exhibits the strongest robustness against NLOS noise among all tested methods. It significantly reduces localization error by approximately 37\% to 53\% compared to the GCN-based method under diverse noise conditions, while maintaining comparable computation times.
\end{itemize}

The remainder of this paper is organized as follows. The background of network localization and its problem formulation are introduced in Sec.~\ref{sec: problem formulation}. In Sec.~\ref{sec: static_GCN}, we briefly review the GCN framework for network localization and show its superb localization performance. 
For more general and practical scenarios, the AGNN model is further proposed in Sec.~\ref{sec: dynamic_netloc_GAT}. 
In Sec.~\ref{sec: perfomance_analysis}, we conduct theoretical analyses of AGNN, followed by the numerical results in Sec.~\ref{sec: numerical_results}.
Finally, we conclude the paper in Sec.~\ref{sec: conclusion}. Notations used in this paper are summarized in Tab.~\ref{tab:notations}.

\begin{table}[t]
\caption{Summary of notations.}
\label{tab:notations}
\centering
\begin{tabular}{c|c}
\hline
Notation                           & Definition                                 \\ \hline
$\mathcal{S}$, $|\mathcal{S}|$     & set, cardinality of set                    \\
$\bbx$, $\bbX$, $(\cdot)^\top$     & vectors, matrices, vector/matrix transpose \\
$\bbx_i$, $\bbx_{[i,:]}$, $x_{ij}$ & column, row, element of matrix $\bbX$      \\
$[N]$                              & set of natural numbers $1, \ldots, N$      \\
$[\cdot\|\cdot]$                   & concatenation operation for two vectors    \\
 $\lvert\bbx\rvert$                & element-wise absolute value of vector $\mathbf{x}$ \\
$\|\cdot\|$, $\|\cdot\|_F$         & Euclidean norm, Frobenius norm             \\ \hline
\end{tabular}%
\end{table}


\section{Background and Problem Formulation}
\label{sec: problem formulation}

\subsection{Localization Scenario}

Fig.~\ref{fig:localization_scenario} illustrates a classic wireless network localization scenario. In the context of a massive wireless network, only a limited number of nodes, known as anchors, possess precise location information from satellite navigation systems.
Conversely, the remaining nodes, referred to as agents, lack satellite navigation access and consequently need to be located. Equipped with omni-directional antennas, each node broadcasts signals containing information such as node ID, transmit power, transmit time, and network configurations.

This scenario is highly representative of many real-world massive wireless networks, such as 5G networks, IoT networks, vehicle networks, and metaverse \cite{yin2020fedloc, ning2023survey}. 

\begin{figure}[t] 
	\centering
	\includegraphics[width=0.9\linewidth]{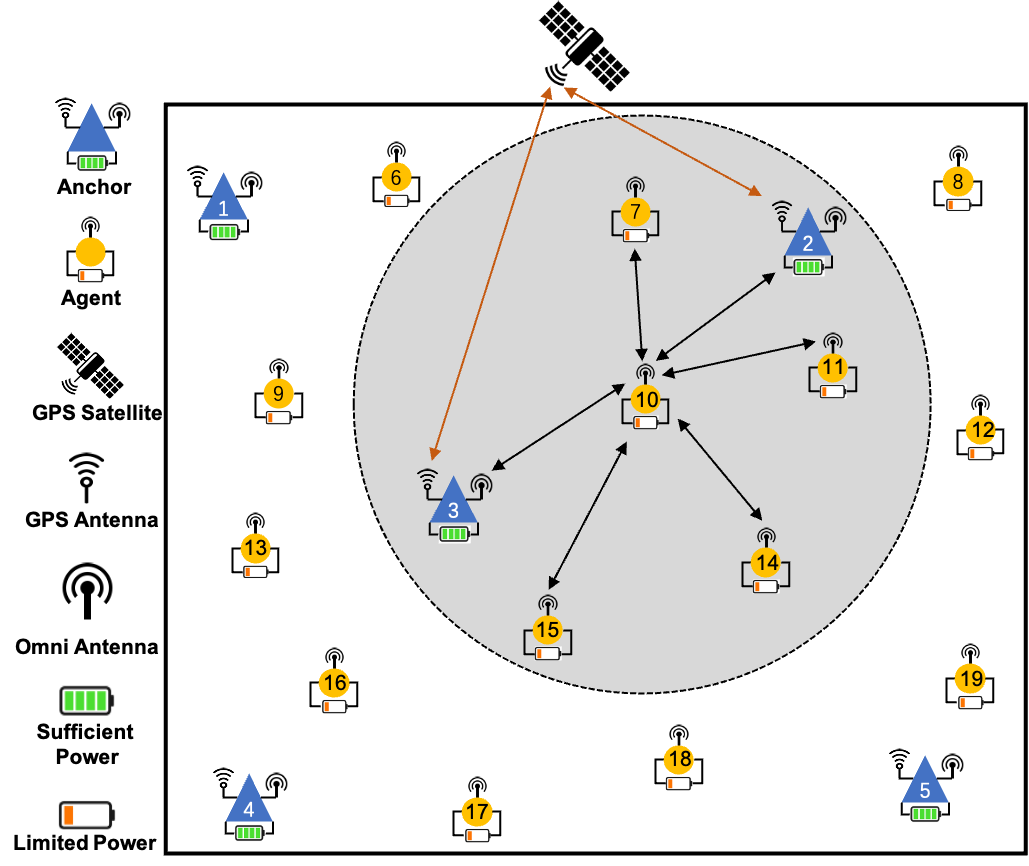} 
	\caption{Network localization scenario. The shaded area represents the region of interest for the $10$-th node after threshold selection which will be introduced in Sec.~\ref{sec: static_GCN_reg}.}
	\label{fig:localization_scenario}
\end{figure}

\subsection{Problem Formulation}
We restrict our study to a wireless network in a two-dimensional (2-D) space, as the extension to the 3-D case is straightforward. We let $\mathcal{S}_l = \{1,2,\dots, N_l\}$ be the set of indices for the anchors (labeled nodes), whose positions $\mathbf{p}_i\in\mathbb{R}^2, \forall i\in \mathcal{S}_l$ are known and fixed, and $\mathcal{S}_u = \{N_l+1, N_l+2,\ldots, N\}$ be the set of indices for the agents (unlabeled nodes), whose positions are unknown. 

In accordance with numerous established studies, e.g., \cite{Chen_Wang_So_Poor_2012}, the measured distance between any two nodes $i$ and $j$ can be represented as follows:
\begin{equation}
	\label{dist_vs_pos}
	x_{ij} = d_{ij} + n_{ij},
\end{equation}
where $d_{ij}:=\|\mathbf{p}_i-\mathbf{p}_j \|$ is defined as the Euclidean distance between nodes $(i,j)$. The term $n_{ij}$ accounts for an additive measurement error arising from line-of-sight (LOS) and non-line-of-sight (NLOS) propagation,  which can be further detailed as follows:
\begin{equation}
\label{Eq:noise_generated}
	n_{ij} = n^L_{ij} + b_{ij}n^N_{ij}.
\end{equation}
Herein, the LOS noise, $n^L_{ij}$, follows a zero-mean Gaussian distribution, i.e., $n^L_{ij}\sim\mathcal{N}(0, \sigma^2)$, while the NLOS term, $n_{ij}^N$, is generated from a positive biased distribution; $b_{ij}$ is sampled from the Bernoulli distribution $\mathcal{B}(p_B)$ with $p_B$ being the NLOS occurrence probability \cite{qi2006analysis}. Consequently, a measurement matrix, denoted by $\mathbf{X}\in\mathbb{R}^{N \times N}$, can be constructed by stacking the measured distances, i.e., the $(i,j)$-th entry of $\mathbf{X}$, $x_{ij}$. Notably, we consider a symmetric scenario where $x_{ij} = x_{ji}$ and $x_{ii} = 0$ for $i = 1 , 2,\ldots, N$.

The measured distances can be estimated from the Received Signal Strength (RSS), or the Time of Arrival (ToA) measurements, which can be acquired through diverse technologies, including cellular networks, wireless local area networks, ultra-wideband radio frequencies, as well as ultra and audible sound. In all cases, NLOS propagation can introduce either significant time delays in ToA measurements or substantial signal attenuation in RSS measurements \cite{zekavat2011handbook}, thereby impeding the accuracy of the eventual distance estimation.

Building on the aforementioned formulations, our primary objective is to precisely localize the agents within a massive wireless network while ensuring satisfactory computation time. To achieve this, we propose a novel GCN-based network localization method in the subsequent section.

\section{Network Localization with Classic GCN}
\label{sec: static_GCN}

In this section, we first introduce a GCN-based data-driven method for network localization that was originally proposed in \cite{yan2021graph}. Subsequently, we provide some key findings regarding the reasons behind the exceptional performance of this model. Lastly, we discuss some of its limitations that necessitate further improvements.

\subsection{Graph Convolutional Network \cite{yan2021graph}}
\label{sec: static_GCN_reg}

In this subsection, our focus lies on the formulation of the network localization problem utilizing classical GCNs.
We formally define an undirected graph associated with a wireless network as ${\mathcal{G}} = ({\mathcal{V}}, \mathbf{A})$, where $\mathcal{V}$ denotes the set of nodes $\{v_1,v_2, \ldots, v_N\}$ in the network, including anchors and agents, and $\mathbf{A}\in\mathbb{R}^{N \times N}$ is a symmetric binary adjacency matrix. 
In classical GCNs, each element $a_{ij}$ of $\mathbf{A}$ represents the existence of the edge connecting nodes $v_i$ and $v_j$, which can be inherently interpreted as a measure of similarity between nodes $i$ and $j$. While, in the context of network localization, the concept of similarity is particularly relevant to proximity, essentially indicated by the Euclidean distance $d_{ij}$ between nodes. To integrate this concept into the GCN-based method, we introduce an Euclidean distance threshold, denoted as $T_h$, to determine the existence of an edge between any two nodes, i.e., whether two nodes are proximal. Therefore, by applying this threshold, the adjacency matrix $\mathbf{A}$ is constructed as follows:
\begin{equation}\label{eq:threshold}
	a_{ij}=\begin{cases}
		0,& \mathrm{if} \quad  x_{ij}>T_h,\\
		1,& \mathrm{otherwise}.
	\end{cases}
\end{equation}
As further elaborated in Sec.~\ref{sec: static_GCN_analysis}, this threshold plays a crucial role in localization performance. And we also need to note that $\mathbf{A}$ includes self-connections, making it an augmented adjacency matrix. Based on aforesaid definitions, we can construct a truncated measurement matrix $\hat{\mathbf{X}}=\mathbf{A} \odot \mathbf{X}$ which only contains those measured distances that are smaller than or equal to $T_h$.
\footnote{In $\hbX$, a zero means a null value, representing distances between corresponding nodes whose values exceed $T_h$.}

With the well-defined graph, we propose our GCN-based method as follows. We first denote the node representation at the $k$-th layer as $\mathbf{H}^{(k)}\in\mathbb{R}^{N \times D_k}$, where $D_k$ is the hidden dimension of the $k$-th layer. And the initial node representation is set to be the truncated measurement matrix, i.e., $\mathbf{H}^{(0)}=\hat{\mathbf{X}}$. In each layer, a \textit{feature propagation} step is implemented first by the following update rule:
\begin{equation} 
	\label{eq:update_matrix}
	\bar{\mathbf{H}}^{(k)}\in\mathbb{R}^{N \times D_{{k}}} = \hat{\mathbf{A}} \mathbf{H}^{({k})},
\end{equation}
where $\hat{\mathbf{A}}\in\mathbb{R}^{N \times N} :={\mathbf{D}}^{-\frac{1}{2}} {\mathbf{A}} {\mathbf{D}}^{-\frac{1}{2}}$ signifies the normalized adjacency matrix \cite{kipf2016semi}, ${\mathbf{D}}\in\mathbb{R}^{N \times N}:= \mathrm{diag}(d_1,d_2, \ldots, d_N)$ is the associated degree matrix of $\mathbf{A}$ with $d_i = \sum_{j = 1}^N a_{ij}$, and $\bar{\mathbf{H}}^{(k)}$ represents the {propagated} representation matrix in the $k$-th graph convolution layer. Conceptually, this step effectively smoothens the {node} representations within the local graph neighborhood, thereby promoting similar predictions for closely connected nodes. Following the feature propagation step, the subsequent two stages of our GCN, namely \textit{linear transformation} and \textit{nonlinear activation}, are implemented as:
\begin{equation} 
	\label{eq:gcn_propagation}
	\mathbf{H}^{({k+1})} \in\mathbb{R}^{N \times D_{k}} =  \varphi \left( \bar{\mathbf{H}}^{(k)} \mathbf{W}^{(k)}\right),
\end{equation}
where $\mathbf{W}^{(k)}\in\mathbb{R}^{D_{{k}} \times D_{k+1}} $ is the trainable weight matrix in the $k$-th layer, and $\varphi(\cdot)$ is an element-wise nonlinear activation function such as $\mathrm{ReLU}(\cdot)= \max(0,\cdot)$ \cite{ramachandran2017searching}.

Taking the output of the GCN as the estimated positions of all nodes, denoted by $\hat{\mathbf{P}}=[\hat{\mathbf{p}}_1,\hat{\mathbf{p}}_2,\dots,\hat{\mathbf{p}}_{N_l}, \hat{\mathbf{p}}_{N_l+1}, \dots \hat{\mathbf{p}}_N]^{\top}\in\mathbb{R}^{N\times 2}$, we can optimize the GCN by the following optimization problem:
\setlength{\jot}{1pt}
\begin{equation}
\label{Eq: opt_problem_GCN}
	\begin{aligned}
		\min_{\bbW} \quad  \mathcal{L} :=& \|\mathbf{P}_l-\hat{\mathbf{P}}_l \|^2_F\\[2\jot]
		\mathrm{s.t.} \quad  \hat{\bbP} =& \mathrm{GCN}_{\bbW}(\bbA,\hat{\bbX}).
	\end{aligned}
\end{equation}
The objective function is the Mean Squared Error (MSE) between the true positions of anchors $\mathbf{P}_l = [\mathbf{p}_1,\mathbf{p}_2,\dots,\mathbf{p}_{N_l}]^{\top}$ and their respective estimations $\hat{\mathbf{P}}_l =  [\hat{\mathbf{p}}_1,\hat{\mathbf{p}}_2,\dots, \hat{\mathbf{p}}_{N_l}]^{\top}$. Here, $\bbW$ denotes all trainable matrices in the GCN model which can be optimized by gradient descent techniques, such as Stochastic Gradient Descent (SGD) \cite{bottou2010large} or Adam \cite{kingma2014adam}.

\subsection{Key Findings of GCN} \label{sec: static_GCN_analysis}

The numerical results, as detailed in Sec.~\ref{sec: numerical_results}, illustrate significant enhancements achieved by the GCN-based method compared to traditional benchmark methods. In the subsequent discussion, we identify key factors underpinning this method's outstanding performance, offering insights into its effectiveness for network localization and guiding future improvements. We pinpoint two major factors: 1) the threshold $T_h$ and 2) the normalized adjacency matrix $\hat{\mathbf{A}}$.

\vspace{6pt}
\subsubsection{Effects of Threshold $T_h$}

An appropriate threshold serves two primary functions: a) truncating significant noise and b) preventing over-smoothing.

\vspace{4pt}
\noindent
\textit{- Noise Truncation.}
Thresholding ensures that $\mathbf{A}$, the adjacency matrix, includes only edges $a_{ij}$ where $n_{ij} \leq T_h - d_{ij}$, implying that $\hat{\mathbf{X}} = \mathbf{A} \odot \mathbf{X}$ retains distance measurements $x_{ij}$ under the threshold $T_h$ and accompanied by confined noise. For those measurements with the existence of NLOS noise, this condition can be further expressed as $ n_{ij}^{{N}} \leq T_h - d_{ij} - n_{ij}^{L}$ with $n_{ij}^{L}\sim\mathcal{N}(0,\sigma^{2})$. By applying an appropriate threshold, we not only constrain the distance but also truncate the noise, such that measurements are retained only if two nodes are adjacent and primarily affected by a small or moderate measurement noise.

\vspace{4pt}
\noindent
\textit{- Avoiding Over-smoothing.}
In the extreme case with an excessively large threshold, rendering the graph fully connected with an all-ones adjacency matrix, all rows of the hidden representation matrix $\bar{\mathbf{H}}^{(k)}$ become indistinguishable. The loss of discriminative ability leads to the predicted positions of all nodes converging towards the same. This phenomenon is known as \textit{over-smoothing}. Thus, selecting an appropriate threshold is crucial to prevent such detrimental effects.
\vspace{4pt}

To assess the threshold $T_h$'s impact on localization, we experimented across varying $T_h$ values within three noise settings, illustrated in Fig.~\ref{fig:threshold}. The results delineate five performance regions corresponding to different threshold ranges. In Region I, initial performance issues stem from insufficient graph edges when $T_h$ is set too low, leading to isolated nodes that hinder effective localization. Within Regions II to IV, a counterbalance is observed between the NLOS noise truncation effect (Region II) and the growing number of neighbors per node (Region IV). Region V exhibits a sharp increase in RMSE due to diminished noise filtering and severe over-smoothing issues. For detailed experimental settings and result descriptions, refer to our previous work \cite{yan2021graph}.

\begin{figure}[t] 
\centering
\includegraphics[width=1\linewidth]{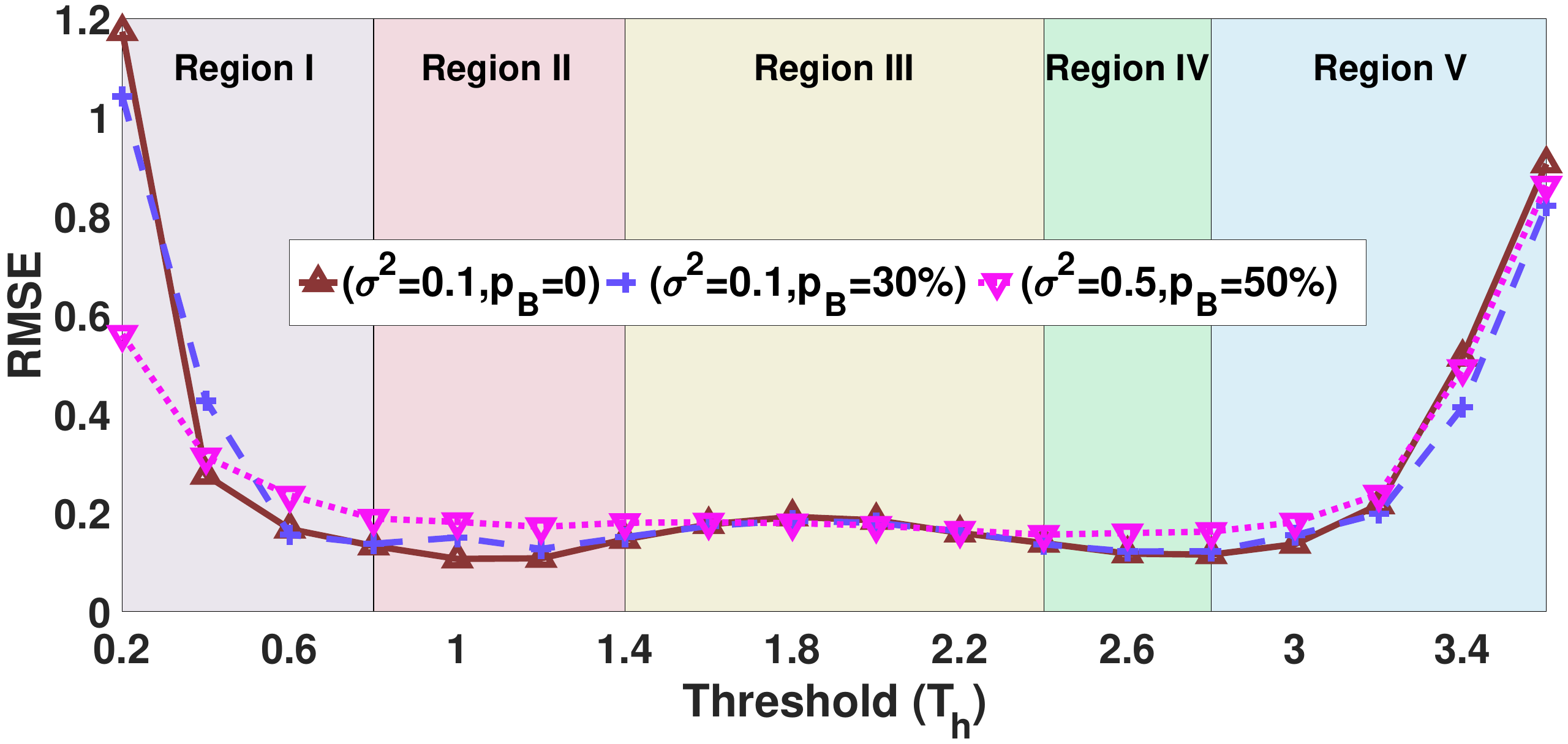}
\caption{The averaged loss (RMSE) v.s. threshold under different noise conditions in the GCN model.}
\label{fig:threshold}
\end{figure}

\begin{figure}[t] 
\centering
\includegraphics[width=1\linewidth]{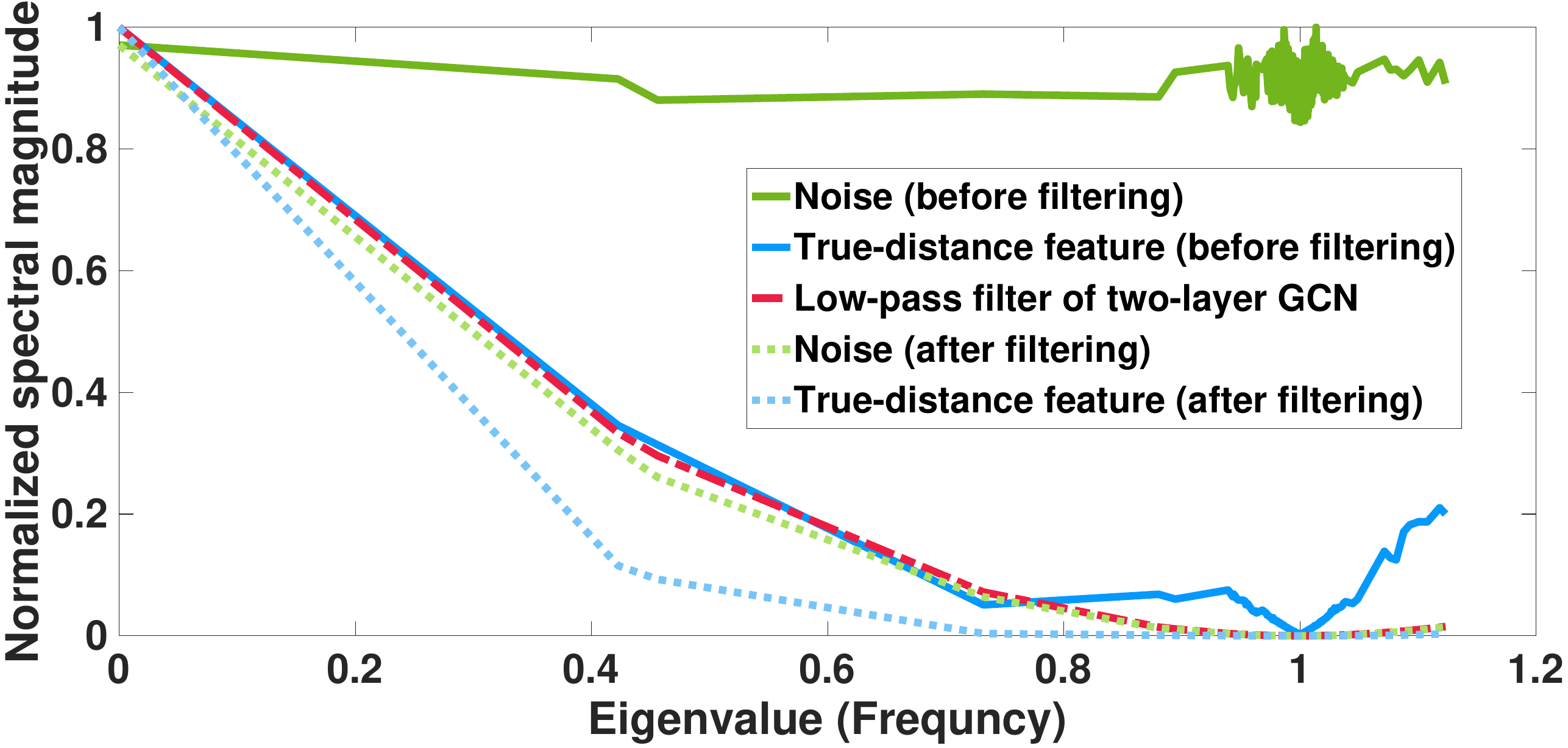}
\caption{Spectral components of the $10$-th column vector (similar for other columns) of the noise, $\bbn_{10}$, and the true distance, $\bbd_{10}$, in dataset $(\sigma^2=0.1,p_B=30\%)$.}
\label{fig:frequency_plot}
\end{figure}

\vspace{6pt}
\subsubsection{Effects of Normalized Adjacency Matrix $\hat{\mathbf{A}}$}

To comprehend the superior localization performance of GCN, we also analyze the effects of the normalized adjacency matrix, $\hat{\mathbf{A}}$, from both spatial and spectral perspectives.

\vspace{4pt}
\noindent
\textit{- Spatial Perspective: Aggregation and Combination.} 
To understand the spatial effect of $\hat{\mathbf{A}}$, we disassemble $\hat{\mathbf{A}}$, cf. Eq.~\eqref{eq:update_matrix}, into two components:
\begin{equation}
	\label{eq: GCN_aggregation}
	\bar{\mathbf{h}}_{[i,:]}^{(k)} = \underbrace{\frac{1}{d_i} \mathbf{h}_{[i,:]}^{({k})}}_{\mathrm{Intrinsic\;information [i]}}+\underbrace{\sum_{j\in\mathcal{N}_i} \frac{a_{ij}}{\sqrt{d_i  d_j}}\mathbf{h}_{[j,:]}^{({k})}}_{\mathrm{Aggregated\;information [ii]}}.
\end{equation}
This equation demonstrates the aggregation of neighboring features (Eq.\eqref{eq: GCN_aggregation}-[ii]) and their combination with a node’s own features (Eq.\eqref{eq: GCN_aggregation}-[i]). As a result, in GCN, the {propagated} representation of each labeled node (i.e., anchor in the localization setting) is a weighted sum of its neighbors. Essentially, GCN leverages the attributes of both labeled and unlabeled nodes in its training process.

\vspace{4pt}
\noindent
\textit{- Spectral Perspective: Low-pass Filtering.} 
From the spectral perspective, the feature propagation step in Eq.\eqref{eq: GCN_aggregation} can be rewritten as
\begin{subequations}
    \begin{align}
        \hat{\mathbf{A}} \mathbf{H}^{(k-1)} &= (\bbI-\hbL)\mathbf{H}^{(k-1)} \\
        &= (\bbI-\bbU\hat{\boldsymbol{\Lambda}}\bbU^\top)\mathbf{H}^{(k-1)} \\
        &= \bbU(\bbI-\hat{\boldsymbol{\Lambda}})\bbU^\top\mathbf{H}^{(k-1)},\label{Eq: graph_filtering}
    \end{align}
\end{subequations}
where $\hat{\mathbf{L}} := \mathbf{I}-\hat{\mathbf{A}}$ denotes the normalized Laplacian, $\bbU$ and $\hat{\boldsymbol{\Lambda}}$ are the associated eigenvectors and eigenvalues. In terms of graph signal operations, $\tbH^{(k-1)}:=\bbU^\top\mathbf{H}^{(k-1)}$ and $\mathbf{H}^{(k-1)}:= \bbU\tbH^{(k-1)}$ correspond to the Graph Fourier Transform and Inverse Graph Fourier Transform, respectively. The vector $\bbu_i$ and the scalar $\hat{\lambda}_i$ represent the $i$-th graph Fourier basis and ``frequency'' component of the graph signal, respectively \cite{sandryhaila2014discrete, gama2020graphs}. According to Eq.\eqref{Eq: graph_filtering}, each feature propagation operation acts as a spectral ``low-pass'' filter $g(\hat{\lambda}_i):=(1 - \hat{\lambda}_i)$.
Consequently, a two-layer GCN can be approximated as $\varphi(\hat{\mathbf{A}}^2 \bbX \bbW)$, which effectively applies the filter $(1 - \hat{\lambda}_i)^2$ \cite{wu2019simplifying}.
This spectral filtering concept is essential for understanding how GCNs process and filter signal information across the network. Fig.~\ref{fig:frequency_plot} illustrates the ``frequency'' components of the noise (including LOS noise and NLOS bias) and the true distance matrix before and after filtering. Notably, in the unfiltered state, the majority of information within the true distance matrix is concentrated in the ``low-frequency'' band. Conversely, the noise displays both ``low-frequency'' and ``high-frequency'' components before filtering. Comparing the results before and after filtering, it becomes evident that the normalized adjacency matrix $\hat{\mathbf{A}}$, serving as a ``low-pass'' filter, effectively mitigates the ``high-frequency'' component of the LOS noise. From a spectral analysis standpoint, this ``low-pass'' filtering capability of $\hat{\mathbf{A}}$ justifies the improved performance of GCNs in localization tasks by filtering out the ``high-frequency'' noise component.

\subsection{Limitations of Classic GCN}
\label{Sec:Limitation_GCN}
Despite the high suitability of the GCN-based method for network localization tasks demonstrated in the preceding discussion, the subsequent discussion reveals several key limitations of this method:

\vspace{4pt}
\noindent
\textit{- Optimal Threshold Challenge.} As shown in Fig.~2, the selection of threshold $T_h$ plays a pivotal role in determining the localization performance. {To achieve optimal results, $T_h$ should be chosen within Region II or Region IV. } However, as a hyperparameter, determining its optimal value for thresholding purpose is not straightforward{, especially when applying the model to new scenarios with varying sizes and shapes}. Furthermore, due to the distinct node positions and varying channel state information, each node may possess its unique optimal threshold, rather than a uniform threshold applicable to all nodes. Consequently, there arises a necessity to develop a more flexible model that is capable of autonomously learning the optimal threshold during the training process.

\vspace{4pt}
\noindent
\textit{- Limited model expressiveness.}
The GCN-based model employs a constant aggregation weight (see Eq.~\eqref{eq: GCN_aggregation}), determined by adjacency and node degree, which constrains its capacity to adaptively aggregate and combine neighborhood information. From the spectral perspective, the GCN-based model represents a fixed graph convolutional filter, as delineated by the green dash-dotted line in Fig.~\ref{fig:frequency_plot}. This fixed filtering characteristic (fixed adjacency matrix) compromises its versatility, necessitating the design of customized filters tailored to network noise characteristics at each graph convolutional layer.

\vspace{4pt}
In the following, we turn our focus to offering effective solutions to mitigate the aforementioned issues.

\section{Network Localization with AGNN}
\label{sec: dynamic_netloc_GAT}

In this section, we first review the attention mechanism used in recent works and its challenges in network localization. Then, we design a fresh Attentional GNN (AGNN) model that can autonomously determine the optimal threshold and adaptively determine the aggregation weight of each node for various network localization tasks.

\subsection{Attention Mechanism for Network Localization}

Motivated by recent advancements in attention-based graph neural networks, such as Graph ATention network (GAT) \cite{velivckovic2017graph} and refined version of GAT (GATv2) \cite{brody2022how}, we incorporate the attention mechanism into the network localization task, enabling the data-driven model to autonomously determine an optimal neighbor set for each node and acquire flexible aggregation weights for all pairs of nodes based on the so-called attention scores.

Graph attention architectures typically consist of a sequence of graph attention layers \cite{velivckovic2017graph}. To elucidate, we start by explaining the basic structure of a single graph attention layer, which can be represented as:
\begin{equation} 
	\label{eq: general_attention}
	e_{ij}^{(k)} = \mathrm{Att}\left(\mathbf{h}_{[i,:]}^{(k)}, \mathbf{h}_{[j,:]}^{(k)}\right).
\end{equation}
Herein, at the $k$-th layer, a shared attentional mechanism, denoted as $\mathrm{Att}(\cdot,\cdot): \mathbb{R}^{D_k} \times \mathbb{R}^{D_k} \rightarrow \mathbb{R}$, calculates an attention score $e_{ij}^{(k-1)}$ for each node pair $i$-$j$, reflecting the relative importance of node $j$'s features to those of node $i$. Essentially, Eq.~\eqref{eq: general_attention} enables each node to determine the importance of all its neighbors through attention scores.

However, directly applying this mechanism to network localization tasks encounters significant hurdles. Specifically, in such tasks when a predefined threshold $T_h$ is absent, i.e., the graph structure is not predetermined, it's common to assume a complete graph for deploying graph attention layers. As a result, the attention mechanism calculates scores for every node pair, leading to two main challenges:

\vspace{4pt}
\noindent
\textit{- High Computational Complexity.}
As the number of nodes increases, the number of possible node pairs grows quadratically. Therefore, computing attention scores for all possible node pairs is computationally expensive. This issue can significantly slow down the training of the attention mechanism, especially when dealing with the localization of a massive network.

\vspace{4pt}
\noindent
\textit{- Over-smoothing Issues.}
Similar to the analysis provided in Sec. \ref{sec: static_GCN_analysis}, in a complete graph, aggregating features from all nodes causes over-smoothing. Even with varied attention scores for edges, this mechanism doesn't effectively mitigate over-smoothing, especially reducing the model's expressiveness exponentially as model depth increases \cite{wu2023demystifying}.
\vspace{4pt}

To tackle these challenges, we introduce a novel AGNN model tailored for network localization tasks in the subsequent sections. This model is explicitly crafted to autonomously acquire graph structures during the training process, making it highly suitable for network localization tasks across diverse network scenarios.

{In addition to our proposed node-level attention-based approach, several other strategies have been explored to mitigate the over-smoothing problem. One such approach is the use of neighbor sampling \cite{ying2018graph, chen2018fastgcn}, which reduces the receptive field of each node, thereby preventing the model from overly depending on any single node and maintaining feature diversity across nodes. Another promising method is the incorporation of residual connections, which has been demonstrated to alleviate over-smoothing by preserving information flow across layers of the model \cite{chen2020simple}. While these methods have been proven effective in various contexts, we argue that the AGNN model introduced here, with its ability to learn graph structures dynamically during training, provides a tailored solution for network localization, where the graph structure is not predefined.}

\subsection{Overview of AGNN Model}
\label{sec: dynamic_GAT_reg}

\begin{figure*}[t] 
	\centering
	\includegraphics[width=1\linewidth]{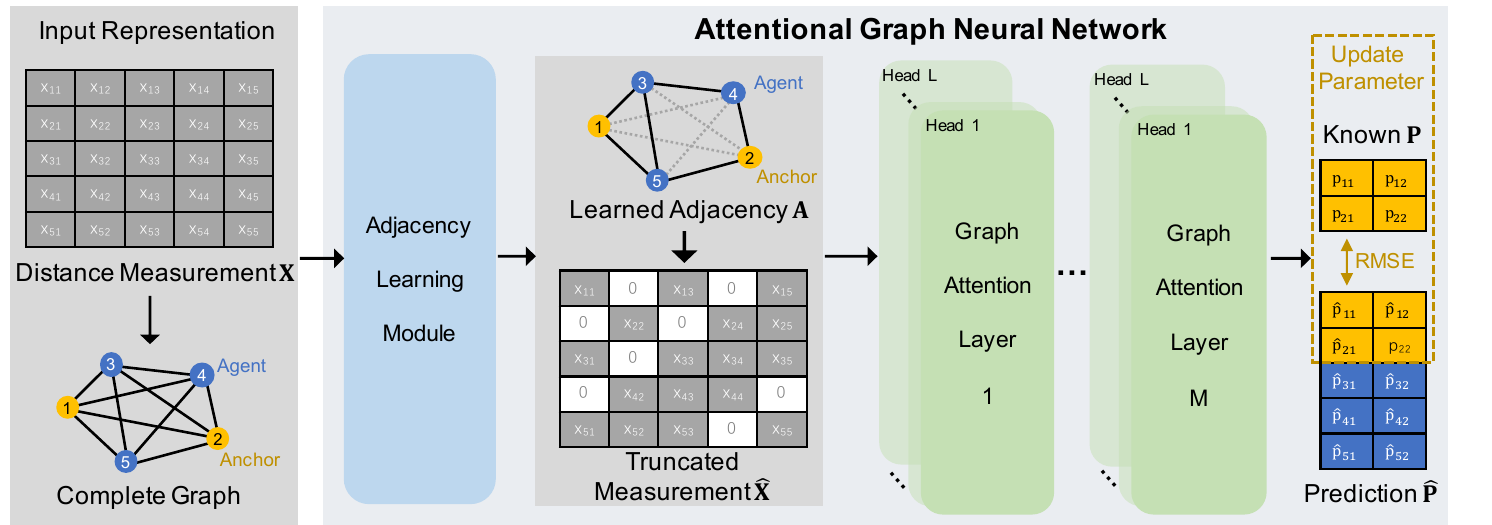}
	\caption{Framework of AGNN.}
	\label{fig: AGNN_framework}
    \hrulefill
\end{figure*}

We propose the AGNN model tailored for network localization tasks, which is composed of an Adjacency Learning Module (ALM) and Multiple Graph Attention Layers (MGAL), as illustrated in Fig.~\ref{fig: AGNN_framework}.
The ALM leverages an attention mechanism to learn a distance-aware threshold. In other words, nodes can adaptively adjust the threshold based on the distance to their neighbors, enabling more flexible neighbor selection. Subsequently, the MGAL utilize another attention-based mechanism to acquire adaptable aggregation weights based on the adjacency matrix learned by ALM.

The ALM autonomously determines the neighbor set for each node based on the attention mechanism.
This approach is able to learn an optimal adjacency matrix selectively, thereby significantly reducing computational complexity and effectively addressing the over-smoothing issue associated with the direct utilization of the graph attention mechanism in a complete graph. Next, we delve into the details of the ALM.

\subsection{Adjacency Learning Module (ALM)}
\label{sec: attentional_neighbor_selection}

To adaptively learn the optimal threshold during the training process, we propose an ALM, which incorporates an attention mechanism to learn a specific threshold $T^{A}_{ij}$ for each node pair. Nevertheless, learning $T^{A}_{ij}$ directly for each pair without prior graph structure knowledge still poses computational challenges. To mitigate this, ALM employs a two-step process: initially, a coarse-grained neighbor selection is conducted using a manually-set threshold, followed by a fine-grained neighbor refinement through an attention mechanism. Details of the two stages are as follows. 

\vspace{4pt}
\noindent
\textit{- Stage I: Coarse-grained Neighbor Selection.}
To address the challenge of computing attention scores without prior knowledge of the graph structure, we initially utilize a preselected threshold, $T_h^0$, to perform a coarse-grained neighbor selection for each node, following the procedure in Eq.~\eqref{eq:threshold}. 
This initial threshold can be manually set or optimized by a process detailed in App.\ref{sec: thre_learning_method}, thereby eliminating the reliance on manual selection. Notably, unlike the threshold in the GCN-based model which can significantly influence the localization accuracy, this initial threshold $T_h^0$ can be chosen from a broader range, as our objective here is to establish a coarse-grained neighbor set, $\mathcal{N}_i^C$, for each node $i$. Verification of this assertion can be found in Fig.~\ref{fig:threshold_new} in Sec.~\ref{sec: numerical_results}. 

\vspace{4pt}
\noindent
\textit{- Stage II: Fine-grained Neighbor Refinement.}
To refine coarse-grained neighbor sets into fine-grained ones, we employ a distance-aware threshold matrix, which is calculated using attention scores. This matrix is employed to further refine the coarse-grained neighbor sets by comparison with the measured distance matrix.

Specifically, we incorporate the masked attention mechanism \cite{vaswani2017attention} into the coarse-grained graph structure. This means that we compute the attention scores $e_{ij}^A$ only for nodes $j$ within the coarse-grained neighbor set $\mathcal{N}_i^C$ of node $i$. Our designed attention scores $e_{ij}^A$ are calculated as follows:
\begin{equation}
	\label{eq: neighbor_selection_att}
	e_{ij}^A =\left| \phi\left( \bbx_{[i,:]}\bbW_A \right) -   \phi\left(\bbx_{[j,:]}\bbW_A \right) \right| \bbv_A,\;  j\in \ccalN_i^C.
\end{equation}
Here, $\bbx_{[i,:]}$ represents the original measured distance row vector. The attention mechanism in this context is parametrized by an attention weight vector $\bbv_A \in \mathbb{R}^{F_A}$ and an attention weight matrix $\bbW_{A}\in \mathbb{R}^{N \times F_A}$, and applies a nonlinear mapping $\phi(\cdot)$, such as LeakyReLU. The reasons for using Eq.~\eqref{eq: neighbor_selection_att} to compute the attention scores, as opposed to a single-layer neural network used in GAT/GATv2, are as follows:
\begin{itemize}[itemsep=0pt, topsep=2pt]
	\item Passing each node's feature through a linear transformation with $\bbW_A$ followed by a nonlinear activation function enhances the expressive power of feature transformations;
	\item Taking the absolute difference between two transformed representations can be regarded as a learnable distance metric between their representations, which possess the symmetry property that is not considered by GAT and GATv2 but is essential for network localization;
    \item The attention vector $\bbv_A$ learns an effective dimensionality reduction mapping to project the learnable distance metric into a scalar, serving as the output attention score.
\end{itemize}

To convert the attention scores into thresholds, which are comparable to the measured distances, we apply nonlinear scaling through:
\begin{equation}
	\label{eq: scale_att}
	T_{ij}^A  = \mathrm{max}(\bbx_{[i,:]})\cdot\mathrm{Sigmoid}({e}_{ij}^A),
\end{equation}
where $\mathrm{max}(\bbx_{[i,:]})$ identifies the maximum value from the row vector $\bbx_{[i,:]}$. Using Eq.~\eqref{eq: scale_att}, we obtain a distance-aware threshold matrix denoted as $\bbT^A\in\mathbb{R}^{N \times N}$.
With the distance-aware thresholds obtained within each coarse-grained neighbor set, we establish a criterion for selecting fine-grained neighbors by comparing the measured distance with the thresholds. Specifically, the adjacency can be represented as follows:
\begin{equation}\label{eq:thre_adj_att}
	a_{ij}=\begin{cases}
		1 ,\quad \mathrm{if} \quad  x_{ij} <T_{ij}^A, \\
		0,\quad \mathrm{otherwise}.
	\end{cases}
\end{equation}

Nevertheless, Eq.~\eqref{eq:thre_adj_att}, formulated as a step function, lacks differentiability concerning $\bbT^A$, which contains the parameters $\bbW_A$ and $\bbv_A$ that we intend to train using gradient-based optimization methods. Typically, conventional approaches often employ a Sigmoid function as an approximation for the step function. However, for our specific objective of achieving a sparse adjacency matrix by eliminating elements exceeding the threshold, this approach is inadequate as the Sigmoid function's transition at the threshold point is too smooth to effectively produce a sparse adjacency matrix. Consequently, it is essential to identify an approximation for Eq.~\eqref{eq:thre_adj_att} that is both differentiable concerning the trainable parameters and capable of filtering out values that exceed the threshold.

Based on the above requirements, we develop an appropriate approximated step function that addresses this issue. The approximated step function can be defined as follows:
\begin{equation}
\label{eq: Approx_step_Fun}
	\hat{\delta}(x) = \mathrm{ReLU}(\mathrm{tanh}(\gamma x)),
\end{equation}
where $\gamma$ is a hyperparameter for adjusting the steepness of the hyperbolic tangent (tanh) function.
Our devised approximated step function, employing ReLU and tanh, yields a notably sharper transition on the positive side and effectively truncates negative values to zero, demonstrating a more accurate approximation of the step function than the Sigmoid function. More importantly, its derivative can be obtained everywhere with respect to the trainable threshold. For more details, see App.~\ref{app: appro_step_func}.
Subsequently, the approximation of Eq.~\eqref{eq:thre_adj_att} can be expressed as:
\begin{align}
	a_{ij} = \mathrm{ReLU}(-\mathrm{tanh}(\gamma(x_{ij}-T_{ij}^A) )).
\end{align}
After obtaining the adjacency matrix $\bbA$ via the attentional method, we construct the truncated measurement matrix $\hat{\mathbf{X}}=\mathbf{A} \odot \mathbf{X}$. Subsequently, they are effectively used as inputs to feed the graph attention layer introduced below.

\vspace{5pt}
\subsection{Multiple Graph Attention Layers (MGAL)}
\label{sec: graph_attention_layer}

{We assume that the adjacency matrix $\bbA$ has been learned via ALM. The fine-grained neighbor set of each node, denoted as $\ccalN_i^F$ for node $i$, is induced by $\bbA$.}
Subsequently, we introduce MGAL, a more generalized form of the graph attention mechanism inspired by GATv2 \cite{brody2022how}, to learn the aggregation weight and predict the positions of all nodes.

We begin by explicating the $k$-th individual graph attention layer, which functions as the foundational layer uniformly employed throughout the entire framework of MGAL. The node representation at the $k$-th layer is denoted as $\mathbf{H}^{(k)}\in\mathbb{R}^{N \times D_k}$, where $D_k$ is the hidden dimension of the $k$-th layer.
{The initial node representations are set as the truncated measurement matrix: $\mathbf{H}^{(0)} = \hat{\mathbf{X}}$.}
Here, we employ a single-layer feed-forward neural network, parametrized by an attention weight vector $\bbv_{att}^{(k)} \in \mathbb{R}^{F_{att}^{k}}$ and an attention weight matrix $\bbW_{att}^{(k)}\in \mathbb{R}^{2D_{{k+1}} \times F_{att}^{k}}$, as the attention mechanism $\mathrm{Att}(\cdot, \cdot)$. Then, the computation of the attention scores using Eq.~\eqref{eq: general_attention} can be expressed as:
\begin{equation}
    \label{eq: dynamic_GAT_att}
e_{ij}^{(k)} =\phi\left( \left[ {\hbh}_{[i,:]}^{(k)} \left\| {\hbh}_{[j,:]}^{(k)}\right.  \right]  \bbW_{att}^{(k)}  \right)  \bbv_{att}^{(k)},
\quad {j\in \ccalN_i^F,}
\end{equation}
where
\begin{equation}
\label{eq: feature_trans}
\hbh_{[i,:]}^{(k)} =\bbh_{[i,:]}^{(k)} \bbW^{(k)},\quad
\hbh_{[j,:]}^{(k)} =\bbh_{[j,:]}^{(k)} \bbW^{(k)}.  
\end{equation}
Here, ${\bbW^{(k)} }\in \mbR^{D_{k}\times D_{k+1}}$ denotes the weight matrix in the ${k}$-th graph attention layer. 

Notably, we use two distinct learnable weight matrices, $\bbW^{(k)}$ and $\bbW_{att}^{(k)}$, with distinct functionalities in our approach. Concretely, $\bbW^{(k)}$ in Eq.~\eqref{eq: dynamic_GAT_att} functions as a linear transformation, converting input representations into higher-level representations. In contrast, $\bbW_{att}^{(k)}$ collaborates with the weight vector $\bbv_{att}^{(k)}$ to establish an attentive rule, determining the correlation between any selected pair of nodes. Moreover, the proposed graph attention layer represents a more generalized form than GATv2 and can be reduced to GATv2 when $F_{att}^{k}=2D_{{k+1}}$ and $\bbW_{att}^{(k)}$ assumes an identity matrix.

To make the attention scores comparable across different nodes, we normalize them across all choices of $j$ using the $\mathrm{softmax}$ function:
\begin{equation}
	\label{eq: softmax_GAT_att}	
	\alpha_{ij}^{(k)} = \frac{\exp({e}_{ij}^{(k)} )}{\sum_{j'\in\ccalN_i^F} \exp({e}_{ij'}^{(k)} )}.
\end{equation}
Once obtained, the normalized attention scores $\alpha_{ij}^{(k)}$ are used to update $\bbh^{(k+1)}_{[i,:]}$ with nonlinearity $\varphi(\cdot)$ as:
\begin{equation}
\label{eq: dynamic_GAT_eqnatt}
	\bbh^{(k+1)}_{[i,:]} = \varphi\left(\sum_{j\in\ccalN_i^F} \alpha_{ij}^{(k)} \hbh_{[j,:]}^{(k)} \right).
\end{equation}
Analogous to the feature updating mechanism in each graph convolutional layer of the GCN-based method, Eq.~\eqref{eq: dynamic_GAT_eqnatt}  
can be decomposed into \textit{feature propagation}, \textit{linear transformation} and \textit{nonlinear activation}. The primary distinction lies in the \textit{feature propagation}: the GCN-based method employs a constant aggregation weight, whereas MAGL utilize learned aggregation weights through the attention mechanism.

In summary, we have proposed an ALM capable of autonomously learning optimal thresholds for individual edges through an attention mechanism. This ALM further leads to an adaptive adjacency matrix and the corresponding truncated measurement matrix. Subsequently, MGAL leverages these learned matrices and employs another attention mechanism to acquire adaptable aggregation weights, enhancing the expressiveness of the model. As a result, it effectively addresses the limitations of the classic GCN, as discussed in Sec.~\ref{Sec:Limitation_GCN}. 

Finally, we formulate the optimization problem for our proposed AGNN as follows:
\setlength{\jot}{1pt}
\begin{equation}
		\begin{aligned}
		\min_{\bbW_1,\bbW_2}~~~~~ \mathcal{L} &= \|\mathbf{P}_l-\hat{\mathbf{P}}_l \|^2_F\\[2\jot]
		\mathrm{s.t.} ~~~~~ \hat{\bbP} &= \text{MGAL}_{\bbW_2}(\bbA,\hat{\bbX})\\[\jot]
		\mathbf{A},\hat{\bbX} &= \text{ALM}_{\bbW_1}(\bbX, T_h^{0}).
	\end{aligned}
\end{equation}
Here, $\bbW_1$ and $\bbW_2$ denote all trainable matrices in ALM and MGAL, respectively. Similarly to Eq.~\eqref{Eq: opt_problem_GCN}, this optimization problem can be addressed through the application of gradient descent techniques  \cite{bottou2010large, kingma2014adam}.

\section{Performance Analysis}
\label{sec: perfomance_analysis}
In this section, we analyze our proposed AGNN model from the perspectives of two fundamental properties: dynamic attention property and complexity analysis. Through theoretical analysis, we aim to elucidate the reasons behind the superior performance of our AGNN in terms of model representational capacity and computational complexity.

\subsection{Dynamic Attention Property}
\label{dynamic_att_property}

In general, attention serves as a mechanism to compute a distribution over a set of input key vectors when provided with an additional query vector. In the context of network localization, the query vector and input key vectors represent the feature vectors of a target node and its neighbors\footnote{For clarity, we assume a complete graph (fully connected graph) in the subsequent discussions. Analogous results can be extended to the case of an incomplete graph, where the neighbor set of each node becomes a subset of the corresponding set in a complete graph.}, respectively. The attention mechanism in network localization involves the process of learning how to allocate attention weights to all neighboring nodes concerning the target node. In this subsection, we have demonstrated that, in contrast to the classical GAT model \cite{velivckovic2017graph} which is confined to computing static attention and thereby encounters significant limitations, our proposed graph attention mechanism can achieve dynamic attention \cite{brody2022how}.

Static attention \cite{brody2022how}, characterized by consistently assigning the highest weight to the same neighboring node for all target nodes, as visualized in Fig.~\ref{fig:Attention_Comparison} (a). It exhibits limited flexibility and is less effective in network localization scenarios, as it uniformly prioritizes a single node, ignoring relative distance and NLOS conditions. A more reasonable form of attention for network localization is the so-called dynamic attention \cite{brody2022how}. A dynamic attention function can assign the highest weight to the most relevant neighboring node based on each target node’s surrounding physical constraints.

In the classical GAT model \cite{velivckovic2017graph}, a scoring function $e\left(\bbh_i, \bbh_j\right)$ is employed to compute a score for every edge $\left(j,i\right)$, which indicates the importance of neighbor $j$'s feature to the node $i$:
\begin{equation}
	\label{eq: gat}
	e\left(\bbh_{[i,:]}, \bbh_{[j,:]}\right)=\phi\left(\left[\bbh_{[i,:]}\bbW \| \bbh_{[j,:]}\bbW\right]\bbv\right).
\end{equation}
\begin{theorem}
A GAT layer computes only static attention, for any set of node representations $\mathcal{K}=\mathcal{Q}=\{\bbh_{[1,:]},...,\bbh_{[N,:]}\}$.  {Specifically, for every function $e(\cdot,\cdot)$, there exists a “highest scoring” key $j_{m1}\in[N]$ such that for every query $i\in [N]$ and key $j\in[N]$, it holds that $e(\bbh_{[i,:]},\bbh_{[j_{m1},:]})\geq e(\bbh_{[i,:]},\bbh_{[j,:]})$.}
\label{theorem: static_GAT}	
\end{theorem}

\begin{myproof}
See App.~\ref{sec: proof: static_GAT} or refer to \cite{brody2022how}.
\end{myproof}

{Notably, based on the above proof, we can further conclude that in static attention, not only is the key with the highest weight fixed to $j_{m1}$, but the relative order of attention weights for all keys is also fixed, independent of the query. This phenomenon is also illustrated in Fig.~\ref{fig:Attention_Comparison} (a).}
The primary issue with the scoring function given in Eq.~\eqref{eq: gat} for the standard GAT is that the learned weights, $\bbW$ and $\bbv$, are applied in succession, which allows them to be condensed into a single linear layer.

Compared with the standard GAT, each graph attention layer \eqref{eq: dynamic_GAT_att} in MGAL simply applies the attention weight matrix $\bbW_{att}$ after the concatenation, and the attention weight vector $\bbv_{att}$ after the nonlinearity $\phi(\cdot)$.
Next, we show that our proposed graph attention layer addresses the limitation of static attention in GAT and possesses a significantly more expressive dynamic attention property.

\begin{theorem}
	\label{thm: dynamic_attention}
	The graph attention function in Eq.~(\ref{eq: dynamic_GAT_att}, \ref{eq: feature_trans}) computes dynamic attention for any set of node representations $\mathcal{K}=\mathcal{Q}=\{\bbh_{[1,:]},...,\bbh_{[N,:]}\}$. {Specifically, for any mapping $\varphi:[N]\to [N]$, there exists a function $e(\cdot,\cdot)$ such that for every query $i\in [N]$ and key $j_{\neq \varphi(i)}\in[N]$, the property $e(\bbh_{[i,:]},\bbh_{[\varphi(i),:]})\geq e(\bbh_{[i,:]},\bbh_{[j,:]})$ holds.}
\end{theorem}

\begin{myproof}
	See App.~\ref{proof: dynamic_attention}.
\end{myproof}

The core idea in this proof centers around the principle that the graph attention function in Eq.~(\ref{eq: dynamic_GAT_att}, \ref{eq: feature_trans}) can be interpreted as a single-layer graph neural network. Consequently, it can be a universal approximator \cite{hornik1991approximation} of an appropriate function we define, thereby achieving dynamic attention. 

In addition, within the ALM, we have devised a graph attention function, as given in Eq.~\eqref{eq: neighbor_selection_att}, with the specific goal of implementing a distance-aware attention mechanism. In the sequel, we show that this function also possesses the dynamic attention property, a critical aspect contributing to its efficacy in refining neighbors at a fine-grained level.
\begin{theorem}
	\label{thm: dynamic_attention_ALM-II}
	The graph attention function in Eq.~\eqref{eq: neighbor_selection_att} computes dynamic attention for any set of node representations $\mathcal{K}=\mathcal{Q}=\{\bbx_{[1,:]},...,\bbx_{[N,:]}\}$.
\end{theorem}

\begin{myproof}
	See App.~\ref{proof: dynamic_attention_ALM-II}.
\end{myproof}

This proof decomposes Eq.\eqref{eq: neighbor_selection_att} into a summation form and focuses on the $k$-th term. It shows that the maximum value of this term is not fixed but depends on the set $\ccalN^C$ and the $k$-th element of the attention weight vector $\bbv_A$, denoted as $[\bbv_A]_k$. By sorting node representations and considering different cases based on the sign of $[\bbv_A]_k$, the dynamic attention property is established for each term in the summation, thereby ensuring the dynamic attention property for the entire function.

Overall, our proposed graph attention mechanisms in ALM and MGAL inherit the dynamic attention property, making them inherently more expressive than the standard GAT. In the context of network localization, this property ensures the optimal selection of threshold and aggregation weights for node pairs based on their features, effectively addressing the limitations in Sec.~\ref{Sec:Limitation_GCN}.

\subsection{Complexity Analysis}

From a computational perspective, our proposed AGNN exhibits high efficiency. The operations of the graph attention layer in MGAL can be parallelized across all edges, and the computations of output features can similarly be parallelized across all nodes. Notably, these models obviate the need for computationally expensive matrix operations, such as eigen-decompositions, which are typically mandatory in spectral-based methods \cite{bruna2014spectral, defferrard2016convolutional}.

We provide the time complexity analysis for GCN-based localization models as follows:
\begin{theorem}\label{thm: gcn_complexity}
	For the $k$-th  graph convolutional layer, the time complexity can be summarized as $O(ND_{k-1}D_k + |E|D_{k-1})$, where $|E|$ represents the numbers of graph edges constructed based on a given threshold.
\end{theorem}

\begin{proof}
	See App.~\ref{proof: gcn_complexity}.
\end{proof}

Crucially, the time complexity for our proposed AGNN is presented as follows:
\begin{theorem}\label{thm: agnn_complexity}
	The time complexity for ALM is $O(NNF_A+|E^C|F_A)$.
	For the $k$-th  graph attention layer of MGAL, the time complexity can be summarized as $O(ND_{k}F_{att}+ND_{k-1}D_k + |E^F|F_{att}+ |E^F|D_k)$, where $|E^C|$ and $|E^F|$ denote the numbers of edges in the coarse-grained and fine-grained neighbor sets, respectively.
\end{theorem}

\begin{proof}
	See App.~\ref{proof: agnn_complexity}.
\end{proof}

If we make the assumptions that $F_{att}$, $D_{k-1}$, $D_k$, and $N$ are approximately equal and that $|E|$, $|E^C|$, and $|E^F|$ are also roughly equivalent, then the time complexity of our proposed AGNN is comparable to that of the baseline GCN-based model.
It's worth noting that when applying multi-head attention, the storage and parameter requirements increase by a factor of $K$, but the computations of individual heads are entirely independent and can be effectively parallelized.

\section{Numerical Results}
\label{sec: numerical_results}

\begin{table}[t]
	\centering
	\caption{The simulated localization scenarios and computer implementation details.}
	\label{tab:settings}       
    \setlength\tabcolsep{3pt} 
    \footnotesize 
	\begin{tabular}{c|ccccc}
            \hline
		Scenarios  & \multicolumn{2}{c}{Size: 5m$\times$5m} & \multicolumn{3}{c}{\# Nodes (\# Anchors):  500 (20-160)}  \\
		\hline
            \multirow{2}{*}{Configuration}  & \multicolumn{5}{c}{48 Inter Xeon E5-2650 2.2GHz CPUs}\\
                           & \multicolumn{5}{c}{8 NVIDIA TITAN Xp 12GB GPUs}  \\
		\hline
		\multirow{2}{*}{Model}   & 	Layers & Hidden size & Epochs& Learning rate & Dropout  \\
		&   2  & 2000 & 200 & 0.01 & 0.5  \\
		\hline
		\multirow{3}{*}{Threshold}   &  \multicolumn{5}{c}{$T_h=1.2$ for GCN and MLP} \\
		&   \multicolumn{5}{c}{$T_h^0=3.0$ for ALM}  \\
		&   \multicolumn{5}{c}{$T_h=0.6$ for benchmarks}  \\
		\hline
	\end{tabular}
\end{table}

\begin{table}[t] 
	\centering
	\small
	\caption{The averaged loss (RMSE) of all methods under different noise conditions for $N_l=50$.}
	\label{tab:1}
	\renewcommand{\arraystretch}{1.2}
        \setlength\tabcolsep{0.5pt} 
        \footnotesize 
	\begin{tabular}{c|ccccc}
		\hline
		\multicolumn{1}{c|}{\multirow{2}{*}{Methods}} & \multicolumn{5}{c}{Noise $(\sigma^2,p_B)$} \\
		\cline{2-6}
		& $(0.04,0\%)$ & $(0.10,10\%)$ & $(0.25,10\%)$ & $(0.25,30\%)$ & $(0.50,50\%)$  \\
		\hline
		LS\cite{Wymeersch2009}  & 0.2270 & 0.2675 & 0.3884 & 0.4187 & 0.7992\\
            MDS\cite{Costa2006} & 0.2361 & 0.2394 & 0.2822 & 0.7593 & 1.1940 \\
		ECM\cite{Yin_ECM}  & 0.1610 & 0.1857 & 0.3298 & 0.3824 & 0.8011\\
		SDP\cite{jin2021exploiting}  & 0.1171 & 0.2599 & 0.4891 & 0.4641 & 0.9294\\
            {SMILE\cite{clark2023smile}} & {0.3235} & {0.3346} & {0.3797} & {0.5755} & {0.7467} \\
		MLP  & 0.1865 & 0.1769 & 0.2305 & 0.2623 & 0.3358\\
            {GraphSage\cite{hamilton2017inductive}} & {0.1008} & {0.1066} & {0.1124} & {0.1276} & {0.1598} \\
            {GAT\cite{velivckovic2017graph}} & {0.1194} & {0.1178} & {0.1215} & {0.1423} & {0.1783} \\
            {GATv2\cite{brody2022how}} & {0.0899} & {0.0942} & {0.0983} & {0.1107} & {0.1471} \\
        \hdashline
		GCN\cite{yan2021graph}  & 0.1038 & 0.1128 & {0.1006} & {0.1302} & {0.1755}\\
            GCN$_{10000}$ & 0.0771$\downarrow$ & 0.0806$\downarrow$ & 0.0932$\downarrow$ & 0.0951$\downarrow$ & 0.1217$\downarrow$\\
        \hdashline
		AGNN  & \textbf{0.0486} & \textbf{0.0551} & \textbf{0.0638}  & \textbf{0.0812} & \textbf{0.1015}\\
            AGNN$_{10000}$  & 0.0378$\downarrow$ & 0.0489$\downarrow$ & 0.0585$\downarrow$ & 0.0784$\downarrow$ & 0.0960$\downarrow$\\
		\hline
            CRB  & 0.0202 & 0.0344 & 0.0453 & 0.0635 & 0.0941\\
            \hline
	\end{tabular}
\end{table}

\begin{figure}[t]
    \centering
    \begin{subfigure}[b]{0.95\linewidth}
        \centering
        \includegraphics[width=\linewidth]{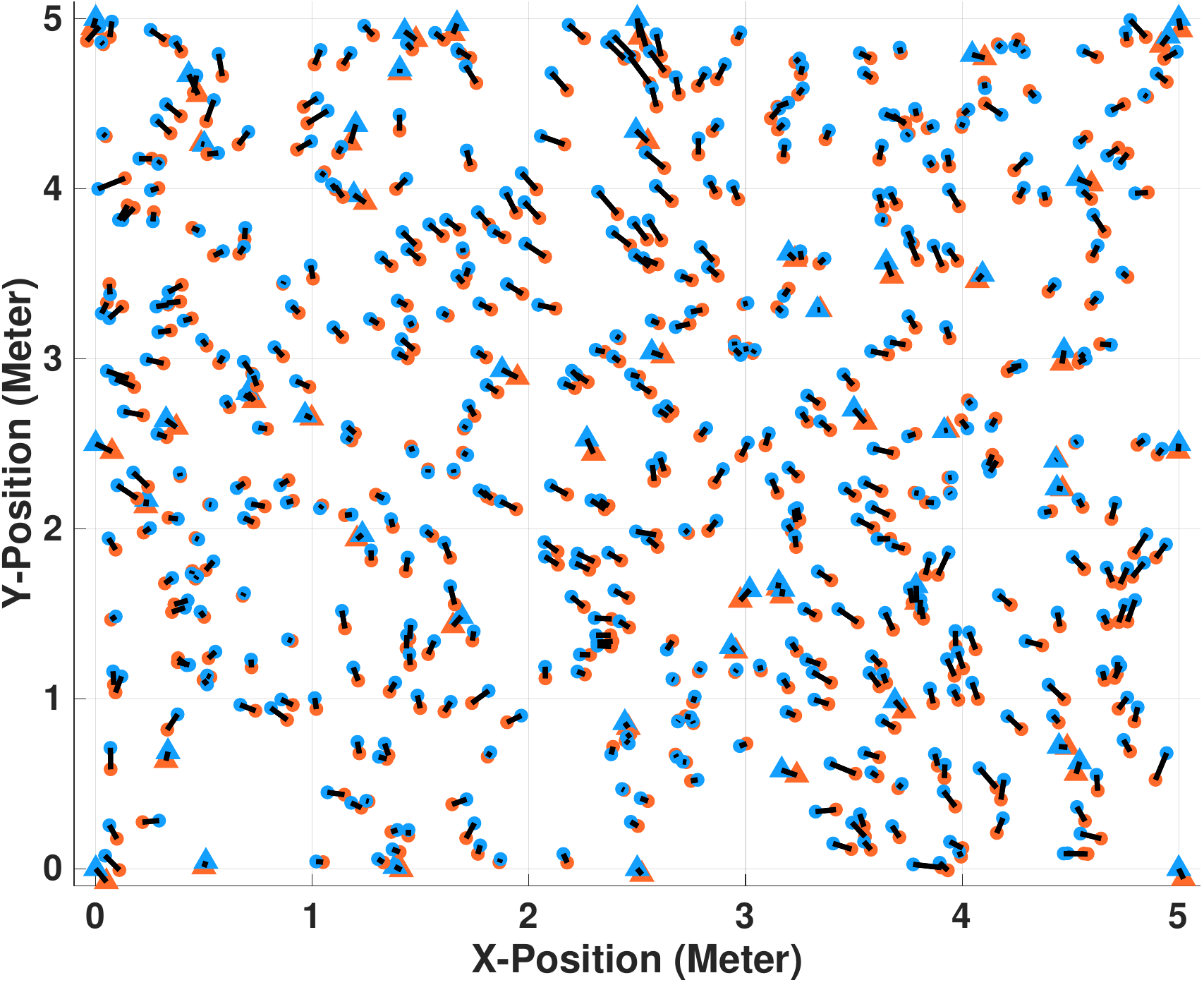}
        \caption{Scatter plot of predicted and true positions. Herein, blue nodes represent the true positions, while orange nodes indicate the predicted positions. Large triangles denote anchor nodes, whereas small circles represent agent nodes. Additionally, black lines depict localization errors between the true positions and those predicted by AGNN.}
        \label{fig: scatter_plot}
    \end{subfigure}

    \vspace{0.5cm} 

    \begin{subfigure}[b]{1\linewidth}
        \centering
        \includegraphics[width=\linewidth]{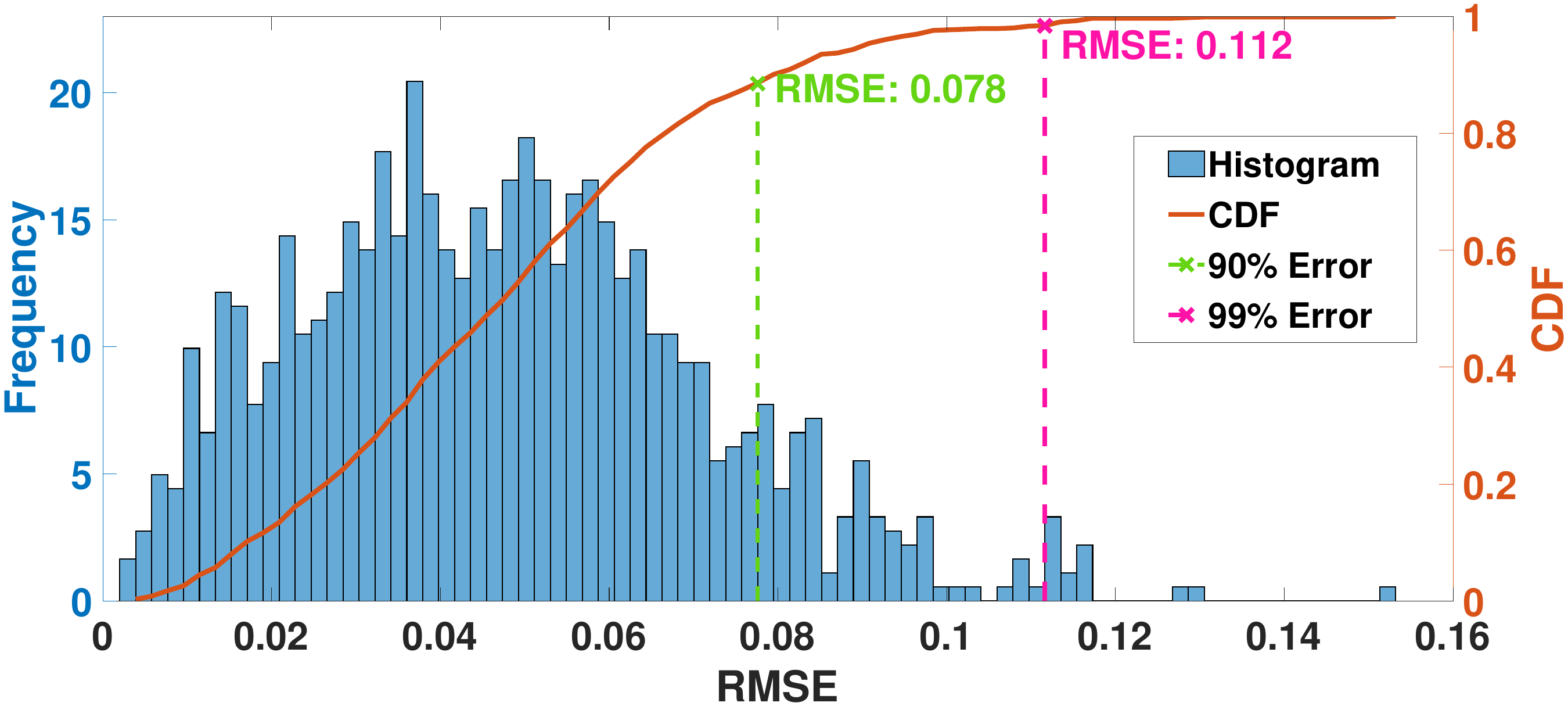}
        \caption{Histogram and CDF of agents' RMSE.}
        \label{fig: histogram_cdf}
    \end{subfigure}
    
    \caption{Visualization of AGNN localization performance with $(\sigma^2=0.1, p_B=0\%)$ and $N_l = 50$.}
    \label{fig: combined_figure}
\end{figure}

In this section, we evaluate the performance of several proposed GNN-based methods in terms of their localization accuracy, robustness against NLOS noise, and computational time. For benchmarks, we employ several classic statistical methods including the sparsity-inducing Semi-Definite Programming (SDP) method \cite{jin2021exploiting}, the Expectation-Conditional Maximization (ECM) method \cite{Yin_ECM}, the Multi-Dimensional Scaling (MDS) method \cite{Costa2006}, the centralized Least-Squares (LS) method \cite{Wymeersch2009}. {We further include the recently proposed SMILE model \cite{clark2023smile}, which leverages low-rank matrix decomposition for network localization.
To complement these comparisons, we evaluate our methods against well-established data-driven baselines, specifically MLP, GraphSAGE\cite{hamilton2017inductive}, GAT\cite{velivckovic2017graph}, and GATv2\cite{brody2022how}.} Note that the purpose of incorporating the MLP-based method, which exclusively employs a \textit{linear transformation} and \textit{nonlinear activation} in each layer, into our comparison is to illustrate the performance enhancements attributed to the inclusion of graph structure in each GNN layer. {Since the adjacency matrix is not directly available in network localization, all GNN-based baselines utilize the same threshold-based adjacency matrix construction strategy as the GCN-based methods.}

Simulated localization scenarios and the computer implementation details are shown in Tab.~\ref{tab:settings}. Here, the measurement error, $n_{ij}$, is generated according to Eq.~\eqref{Eq:noise_generated}. The positive NLOS bias, $n_{ij}^N$, is generated from a uniform distribution\footnote{Additional types of NLOS settings and corresponding results can be found in the App.~\ref{app: analysis_CRB}.}, $n_{ij}^N\sim\mathcal{U}[0, 10]$.
The localization accuracy is measured in terms of the RMSE between estimated positions and true positions of agent nodes, 
$\mathcal{L}_P:=\|\mathbf{P}_u-\hat{\mathbf{P}}_u \|_F$, where $\mathbf{P}_u = [\mathbf{p}_{N_l+1},\mathbf{p}_{N_l+2},\dots,\mathbf{p}_N]^{\top}$ and $\hat{\mathbf{P}}_u =  [\hat{\mathbf{p}}_{N_l+1},\hat{\mathbf{p}}_{N_l+2},\dots, \hat{\mathbf{p}}_{N}]^{\top}$.

Next, we will evaluate the effectiveness of our proposed AGNN, focusing on two primary aspects: its overall superiority concerning accuracy, robustness, and efficiency; and the insights into how the ALM, MGAL, and dynamic attention property contribute to AGNN's exceptional performance.

\subsection{Superiority of AGNN}

\vspace{4pt}
\noindent
\textit{- Overall Performance Analysis.}
Primarily, we conduct a comprehensive evaluation of the localization accuracy across diverse noise conditions and present the results in Tab.~\ref{tab:1}. {Our results reveal consistent superiority of the GNN-based methods (including GCN \cite{yan2021graph}, GraphSage, GAT, GATv2, and newly proposed AGNN) compared to all other benchmarks. 
Notably, their performance in high NLOS environments underscores the robustness and effectiveness of GNN-based localization approaches.
Among the evaluated methods, AGNN achieves the best overall performance across all noise conditions, which highlights the effectiveness of attention mechanisms in learning both adjacency and propagation weights.}
{For better visualization, the predicted positions obtained from the AGNN are depicted in Fig.~\ref{fig: combined_figure} (a) along with the true positions. Additionally, the histogram and Cumulative Distribution Function (CDF) of agents' RMSE for the corresponding scenario are provided in Fig.~\ref{fig: combined_figure} (b).}
To validate the effectiveness of our proposed GNN-based methods in massive networks, we test the localization performance of GCN and AGNN on networks with $N=10000$, denoted by GCN$_{10000}$ and AGNN$_{10000}$, respectively. The results show that both GCN and AGNN perform better, corroborating the positive role played by agents as neighbors in graph aggregation and combination processes.
Ultimately, the AGNN exhibits a minimal performance deviation in comparison to the Cramér-Rao Bound (CRB) \cite{Yin_ECM}, detailed in App.~\ref{app: analysis_CRB}. This observation attests to the AGNN's capability to approach the optimal theoretical performance limits.

\begin{figure}[t] 
	\centering
	\includegraphics[width=1\linewidth]{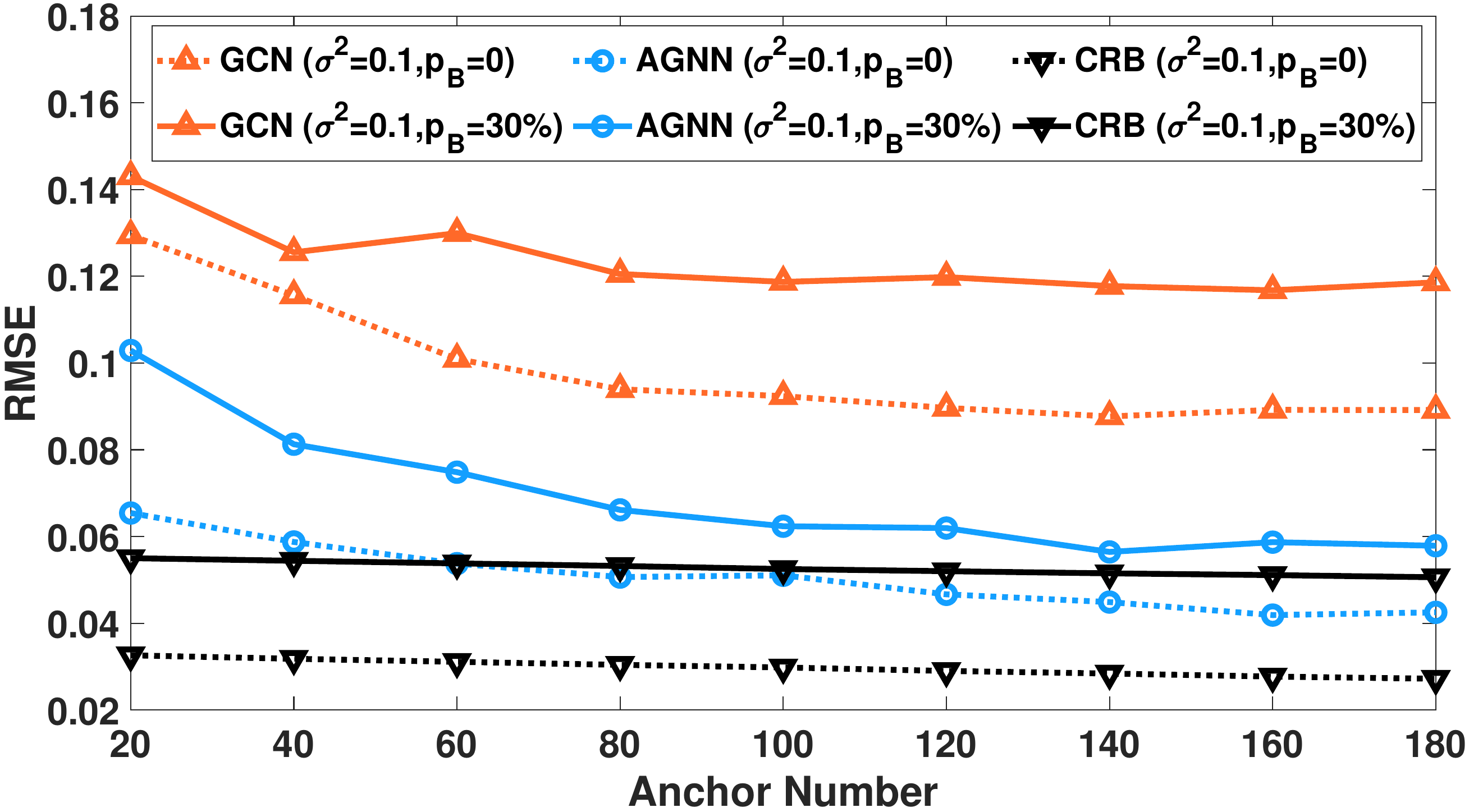}
	\caption{The averaged loss (RMSE) v.s. the number of anchors under different noise conditions.}
	\label{fig:anchor}
\end{figure}

\vspace{4pt}
\noindent
\textit{- Localization Accuracy Analysis.}
Subsequently, we focus on the performance of our proposed GNN-based methods. We explore localization error by varying $N_l$ from $20$ to $180$ with a stepsize of $20$ under two distinct noise conditions, and present the results in Fig.~\ref{fig:anchor}. 
Three key observations can be drawn from these results. 
1) The AGNN consistently achieves significantly lower RMSE for all $N_l$ values compared to the GCN-based method. 
Notably, even under a 30\% NLOS noise condition, AGNN consistently outperforms GCN in the scenarios influenced solely by LOS noise. This underscores the robustness of AGNN to NLOS noise, attributed to learned thresholds and aggregation weights through the attention mechanism.
2) As $N_l$ increases, both AGNN and GCN progressively approach their respective performance limits, indicating the potential of GNN-based methods, employing semi-supervised learning, to achieve accurate localization in massive networks with a limited number of anchors.
3) AGNN's performance limits exhibit a marginal gap compared to the CRB, particularly in NLOS noise conditions. 
This gap arises as the CRB represents the optimal scenario among unbiased methods, assuming complete knowledge of noise distributions and full utilization of node distance information. In contrast, AGNN, while mitigating NLOS noise during the truncation of the measurement matrix $\bbX$, excludes certain distance information, resulting in the observed gap. The reduction of the gap in NLOS scenarios, attributed to AGNN's noise truncation effect, provides additional validation for this observation. Consequently, these findings underscore AGNN's suitability for scenarios influenced by NLOS noise.

\begin{figure}[t] 
	\centering
	\includegraphics[width=1\linewidth]{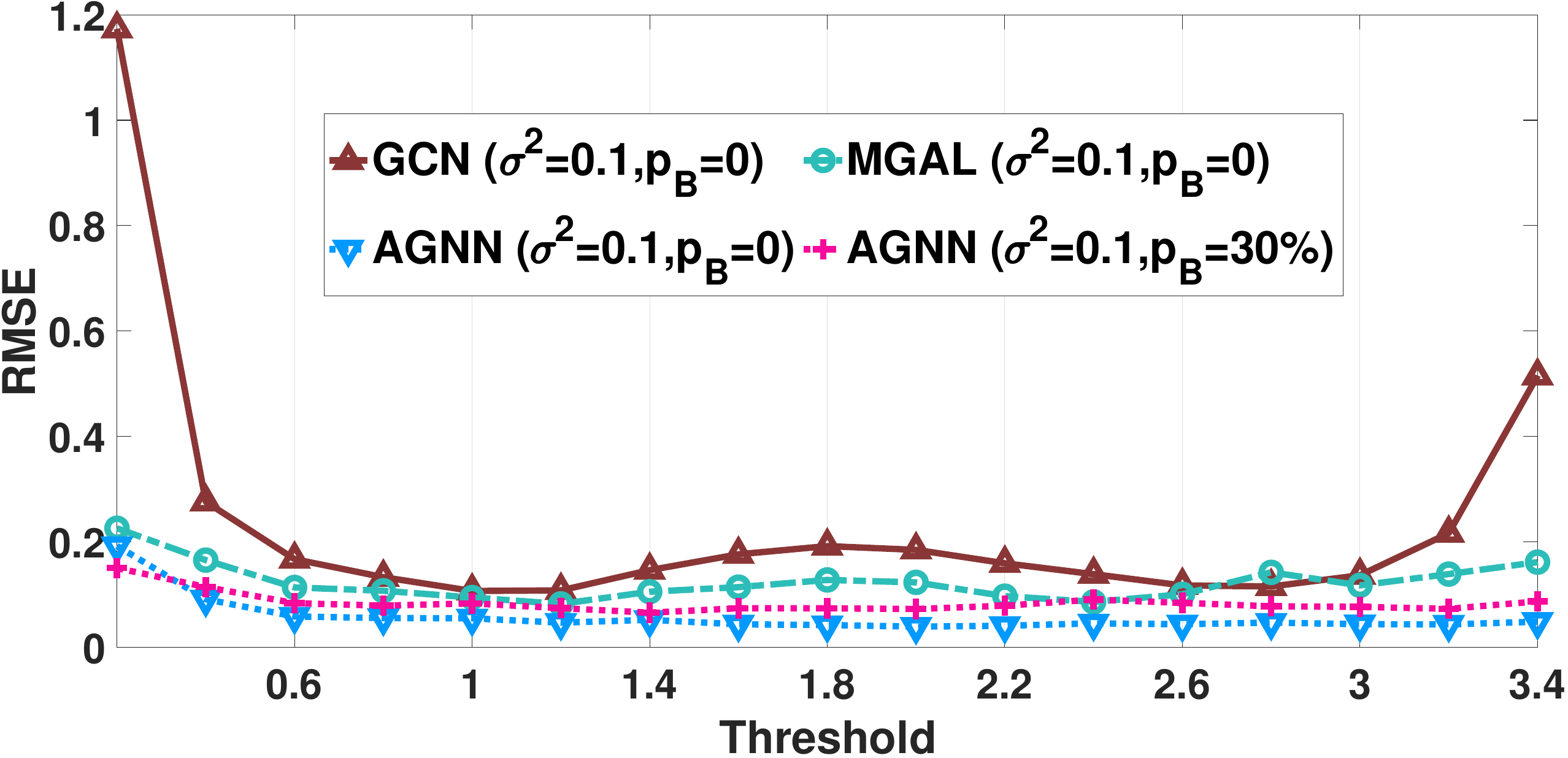}
	\caption{The averaged loss (RMSE) v.s. the threshold ($T_h$ or $T_h^0$) under different noise conditions.}
	\label{fig:threshold_new}
\end{figure}

\begin{table}[t]
\footnotesize
    \centering
        \renewcommand{\arraystretch}{1.2}
    \begin{tabular}{c|c|ccc}
        \hline
        \multicolumn{2}{c|}{{Train: $\ccalD^1$, Test: $\ccalD^1$}} & \multicolumn{3}{c}{{Train: $\ccalD^1$, Test: $\{\ccalD^2,\ccalD^3,\ccalD^4\}$}} \\ \hline
        {Dataset} & {$\ccalD^1$} & {$\ccalD^2$} & {$\ccalD^3$} & {$\ccalD^4$} \\ 
        {Noise} & (0.25, 10\%) & (0.25, 10\%) & (0.25, 10\%) & (0.25, 30\%) \\ 
        {Size} & 5m x 5m & 5m x 5m & 3m x 3m & 5m x 5m \\ \hline
        {GCN} & 0.1006 & 0.1273 & 0.2448 & 0.3689 \\ 
        {AGNN} & 0.0638 & 0.0676 & 0.0783 & 0.2181 \\ \hline
    \end{tabular}
    \caption{{Performance of GCN and AGNN on datasets with varying noise levels, scene sizes, and network realizations}}
    \label{tab: performance_across_dataset}
\end{table}

\vspace{4pt}
\noindent
\textit{- Robustness to Threshold Variation.}
To comprehensively investigate the impact of the threshold, $T_h$, we conduct an experiment systematically exploring the RMSE of GNN-based models across a range of $T_h$, as depicted in Fig.~\ref{fig:threshold_new}. Here, ``MGAL'' denotes a model constructed solely using two graph attention layers from Sec.~\ref{sec: graph_attention_layer}, and subject to a simple adjacency matrix determined by a manually-set threshold. 
For AGNN, we vary the initial threshold $T_{h}^{0}$ employed for the coarse-grained neighbor selection, see Sec.~\ref{sec: attentional_neighbor_selection}.
Three key observations can be drawn from these results. 
1) Both MGAL and AGNN constantly achieve lower RMSE values compared to the GCN-based method utilizing a constant aggregation weight. This underscores the effectiveness of the learned aggregation weights through the attention mechanism in Sec.~\ref{sec: graph_attention_layer}.
2) Akin to MGAL, AGNN exhibits initially low localization error. As $T_h$ increases, MGAL's localization error experiences fluctuations, while AGNN's error decreases rapidly and maintains consistently at a low level. This indicates that for AGNN, the graph attention layer plays a crucial role in ensuring localization accuracy when the initial threshold $T_h^0$ is set to a small value. While $T_h^0$ is relatively large, AGNN resorts to the proposed ALM, which allows for further refinement of the coarse-grained neighbor sets, thereby enhancing the accuracy of localization.
3) AGNN demonstrates consistently low localization errors in two distinct noise environments across various values for $T_h^0$, suggesting the versatility and robustness of the initial threshold choice under diverse noise conditions.

\begin{table}[t]
	\setlength\tabcolsep{3pt}
	\footnotesize
	\centering
	\caption{The computational time (in seconds) of different methods with $(\sigma^2=0.1, p_B=30\%)$ and $N_l = 50$.}
	\label{tab:3}       
	\renewcommand{\arraystretch}{1.3}
	\begin{tabular}{c|cccccccc}
		\hline
		 $N$ & GCN  & 	AGNN  & MLP& {SMILE} & LS & MDS & ECM & SDP   \\

		\hline
		500 & 3.24  & 5.23  & 2.33 & {5.78} & 32.47 & 22.64 & 82.85 & 1587 \\
		1000 & 5.82 & 14.66 & 3.94 & {8.48} & 89.92 & 100.74 & 353.4 & N.A.  \\
		10000 & 707 & 1872 & 212 & {739} & N.A. & N.A. & N.A. & N.A. \\
		\hline
	\end{tabular}
\end{table}

\begin{figure}[t] 
	\centering
	\includegraphics[width=1\linewidth]{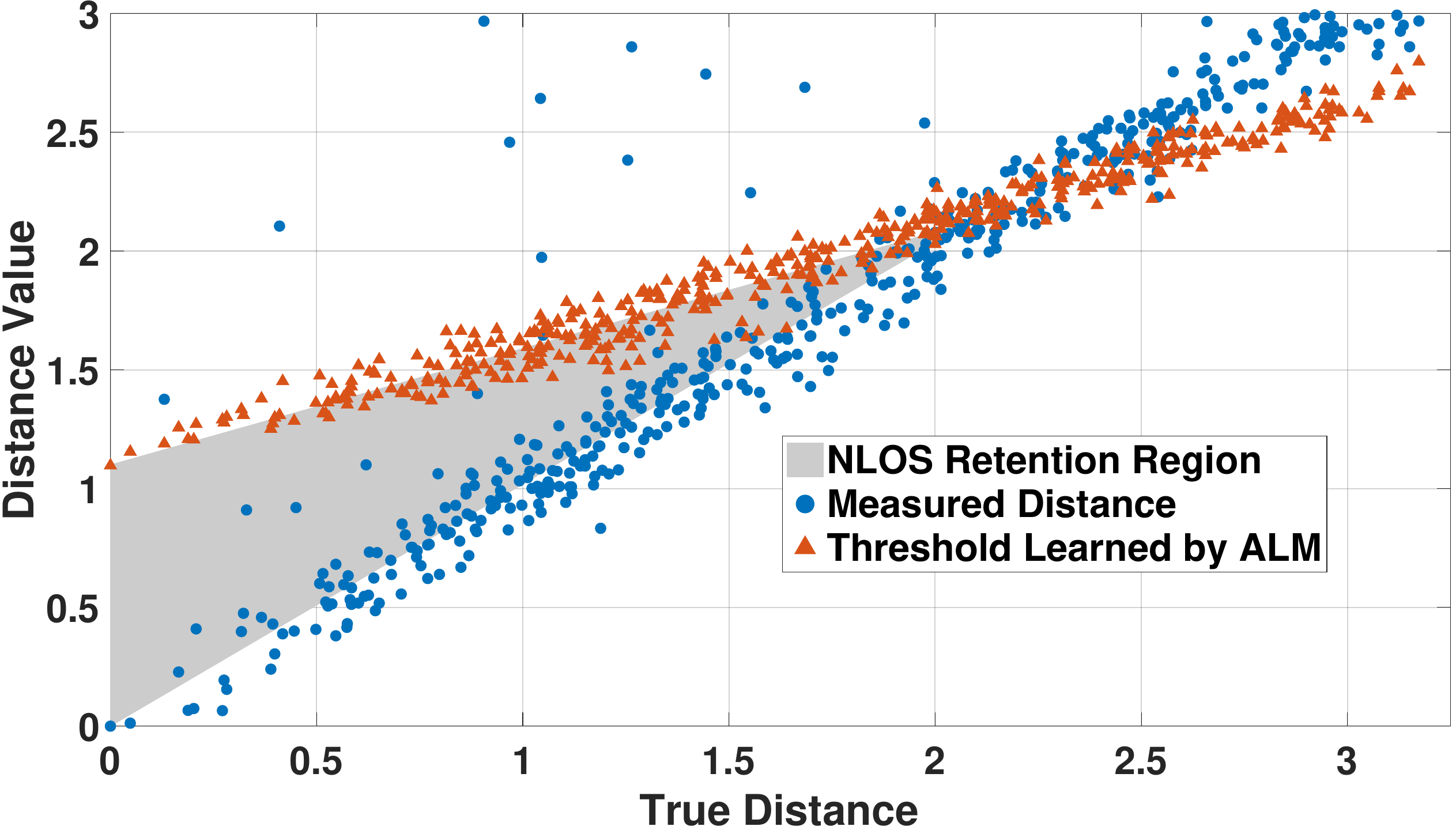}
	\caption{The threshold learned by ALM versus the measured distance for the $10$-th node in $(0.25,30\%)$ dataset with $T_h^0=3.0$.}
	\label{fig: ALMatt_VS_MearDist}
\end{figure}

\begin{figure*}[t] 
	\centering
	\includegraphics[width=0.9\linewidth]{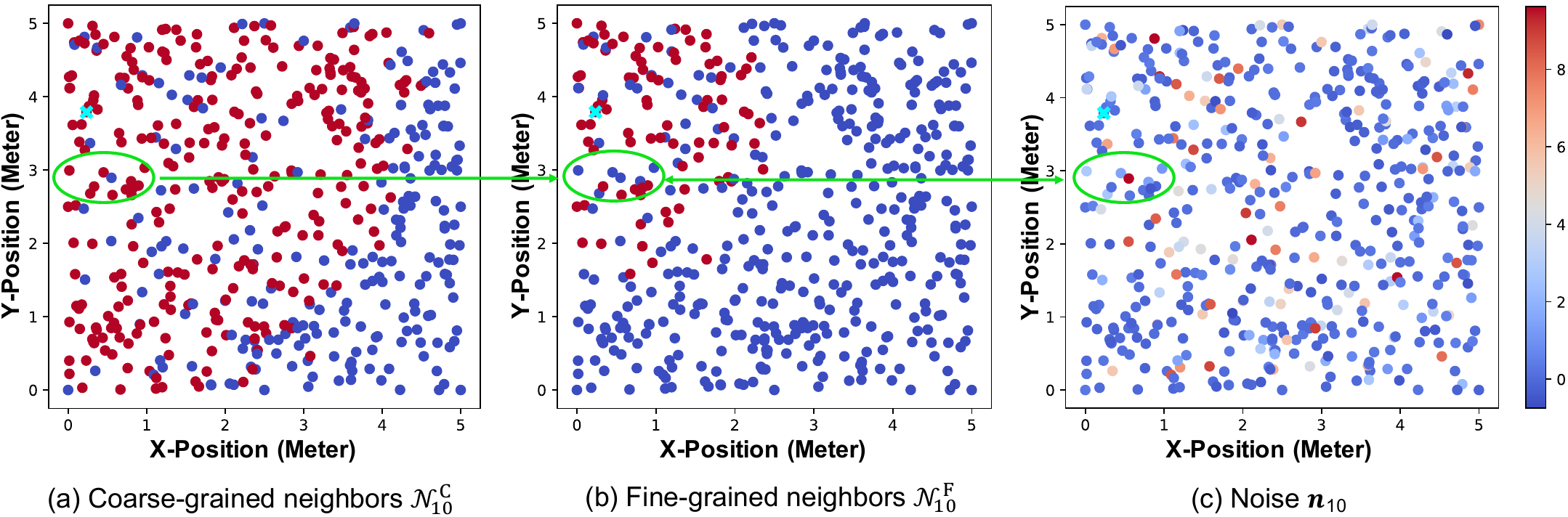}
	\caption{The different types of scatter plots for the $10$-th node, denoted by a cyan cross, in $(0.25,30\%)$ dataset with $T_h^0=4.0$.}
	\label{fig:adj_AdjAGNNII_Noise}
    \hrulefill
\end{figure*}

\setcounter{figure}{9} 
\begin{figure}[t] 
	\centering
	\includegraphics[width=1\linewidth]{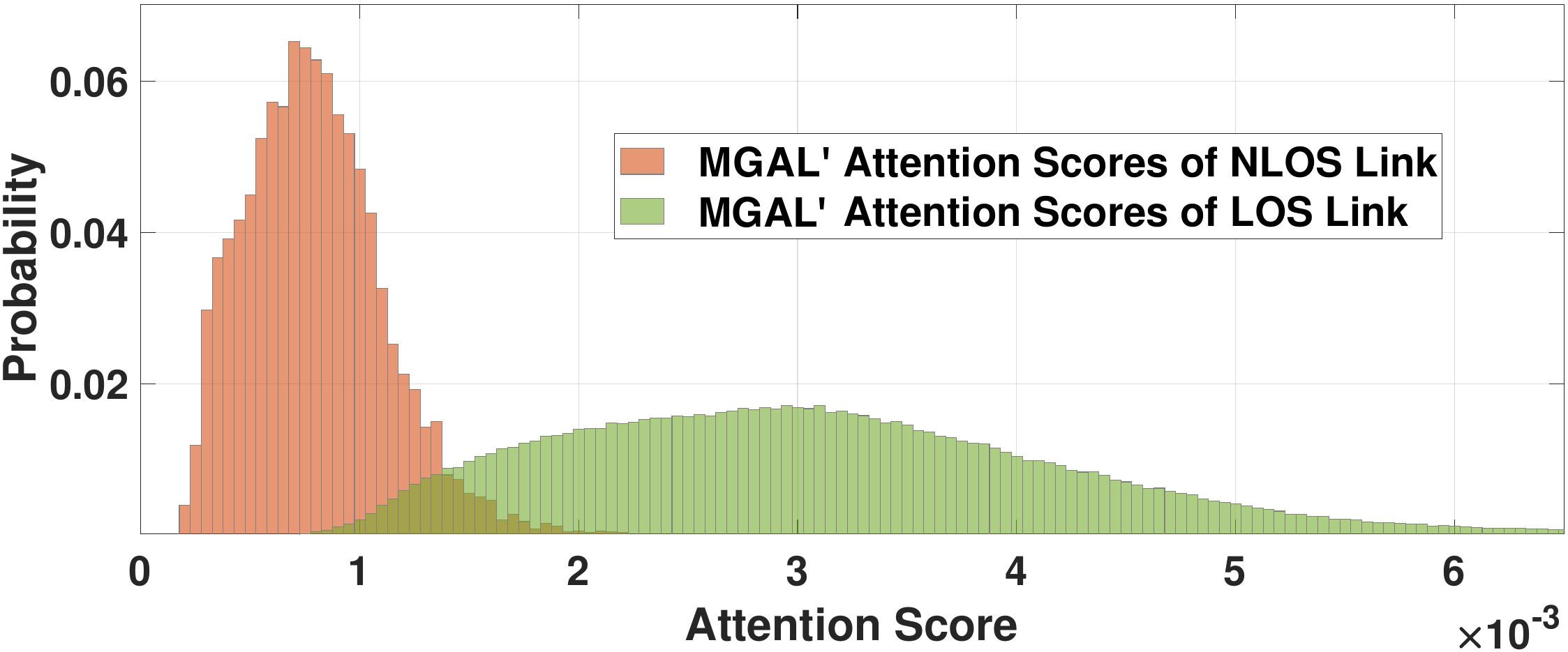}
	\caption{Distributions of MGAL' attention scores for LOS link and NLOS link in $(0.25,30\%)$ dataset.}
	\label{fig: AttScore_VS_Noise}
\end{figure}

\vspace{4pt}
\noindent
{\textit{- Model Generalization.} To demonstrate the generalization and transferability of our GNN-based methods, we evaluate the performance of GCN and AGNN on previously unseen datasets $\{\ccalD^2,\ccalD^3,\ccalD^4\}$ using models trained on $\ccalD^1$, as summarized in Tab.~\ref{tab: performance_across_dataset}. 
Datasets $\ccalD^1$ and $\ccalD^2$ share the same noise distribution and scene size but represent different realizations of the same setting. Both methods exhibit slight increases in localization error when tested on $\ccalD^2$, demonstrating strong robustness to variations in network realizations. In the $\ccalD^3$ scenario, where the scene size changes to $3 \text{m} \times 3 \text{m}$, GCN suffers a notable accuracy decline due to its reliance on fixed thresholds, which fail to adapt to the altered environment. In contrast, AGNN maintains stable performance with only a minor increase in error, as its attention mechanism dynamically recalculates an optimal threshold matrix $\bbT^A$, effectively adapting to the new scene size. 
Finally, in the more challenging $\ccalD^4$ scenario, characterized by a significant increase in NLOS occurrence probability, both methods experience performance degradation. However, AGNN exhibits a smaller decline compared to GCN, highlighting its superior adaptability and robustness across diverse and noisy environments.}

\vspace{4pt}
\noindent
\textit{- Computational Efficiency.}
Beyond accuracy and robustness, real-world applications demand prompt responses. Table~\ref{tab:3} presents the computational times for various methods. Notably, the proposed GNN-based methods exhibit significantly reduced training times compared to traditional model-driven statistical methods. Particularly, for a network with 500 nodes, SMILE, MLP, GCN, and AGNN only require a few seconds, whereas model-driven methods incur substantially higher computational costs, reaching up to 1587s for SDP using commercial software.
Furthermore, it's noteworthy that the computational time for the proposed GNN-based methods only exhibits a modest increase when doubling the number of nodes in the network. This scalability attribute highlights the efficiency of the proposed methods in large-scale scenarios. Especially, for massive scale networks, the proposed GNN-based methods exhibit a superior capability compared to traditional methods, exemplified by the case of $N=10000$ where the GNN-based methods remain computationally affordable, while the traditional methods become impractical at such a scale. These results collectively suggest that the proposed GNN-based methods not only achieve high accuracy but also offer a prospective solution for large-scale network localization with efficient computational performance, making them suitable for real-world deployment.

\subsection{Insights into Well-performed AGNN}

\vspace{4pt}
\noindent
\textit{- NLOS Truncation by ALM.}
To delve deeper into the impact of the threshold matrix $\mathbf{T}^A$ learned by the attention mechanism in our designed ALM, we conduct a scatter plot analysis, as depicted in Fig.~\ref{fig: ALMatt_VS_MearDist}. {This plot compares the learned thresholds for the $10$-th node, $\bbt^A_{[10,:]}$, depicted by orange triangles, with the measured distances, $\bbx_{[10,:]}$, represented by blue circles.
More examples of other nodes can be found in the App.~\ref{app: additional_experiments}. 
The distribution of blue circles reveals that some points deviate significantly from the main trend due to the NLOS effect, which causes the measured distances to be much greater than the true distances. 
In contrast, the distribution of orange triangles leads to the following observations. 
1) The trend of the orange triangles confirms that the attention mechanism effectively learns a distance-related threshold.
2) The filtered-out neighbors (blue circles above the thresholds) fall into two categories: those with short distances but significant NLOS noise (blue circles deviating from the main trend) and those with larger distances.  
3) After applying the distance-related thresholds, the resulting fine-grained neighbor set consists of:  a) Neighbors with short distances and LOS noise (blue circles along the main trend), contributing high-quality information for localization.  
b) Neighbors with short distances and relatively low NLOS noise (blue circles below the thresholds but slightly deviating from the main trend), represented by the gray region in Fig.~\ref{fig: ALMatt_VS_MearDist}. This set will be further refined by MGAL.}

\setcounter{figure}{10} 
\begin{figure*}[t] 
	\centering
	\includegraphics[width=0.9\linewidth]{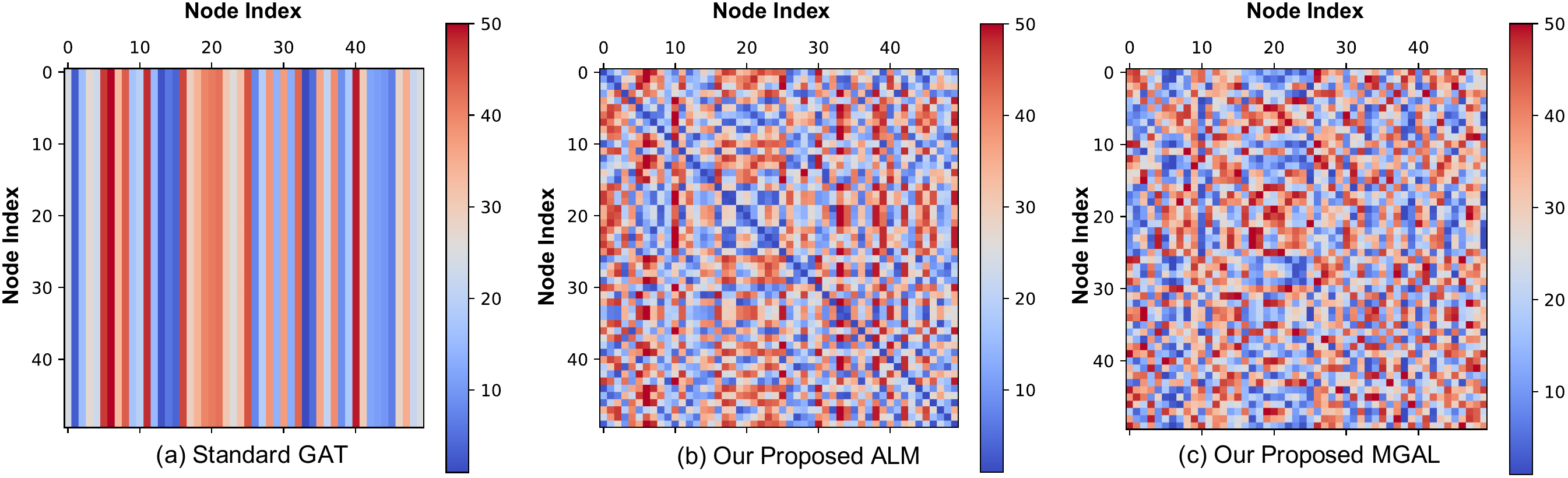}
	\caption{Three heatmaps for attention score rankings using different methods in $(0.25,30\%)$ dataset with $T_h (\mathrm{or}\;T_h^0)=3.0$.}
	\label{fig:Attention_Comparison}
    \hrulefill
\end{figure*}

\vspace{4pt}
\noindent
\textit{- Visualization of ALM's Effect.}
{We further visualize the neighbor selection process within ALM for the $10$-th node using scatter plots and a heatmap. Fig.\ref{fig:adj_AdjAGNNII_Noise} (a) and (b) illustrate the coarse-grained neighbor set $\ccalN_{10}^C$ and the fine-grained neighbor set  $\ccalN_{10}^F$, respectively, with the neighbors highlighted in red and the non-neighbors shown in blue. Fig.\ref{fig:adj_AdjAGNNII_Noise} (c) presents a heatmap of additive noise $\bbn_{10}$, between the $10$-th node and all the others, with varying colors indicating different noise levels.
A comparison of Fig.\ref{fig:adj_AdjAGNNII_Noise} (a) and (b) reveals that the neighbors in $\ccalN_{10}^C$ but excluded from $\ccalN_{10}^F$ primarily fall into two categories: those at relatively large distances from the  $10$-th node, and those close to it but suffering from relatively high noise levels. 
Notably, the cyan circle in Fig.\ref{fig:adj_AdjAGNNII_Noise} (a) highlights close neighbors, some of which are initially included in $\ccalN_{10}^C$ (in red) but are excluded from $\ccalN_{10}^F$ (in blue) in Fig.~\ref{fig:adj_AdjAGNNII_Noise} (b).
Examining the corresponding green region in Fig.~\ref{fig:adj_AdjAGNNII_Noise} (c), we observe that the excluded nodes exhibit higher noise levels (3.0-4.5 meters), confirming  the effectiveness of ALM's attention mechanism in refining $\ccalN_{10}^C$ by removing high-noise neighbors. This refinement reduces noise contamination, enhancing the quality of $\ccalN_{10}^F$ for further processing in MGAL.}

\vspace{4pt}
\noindent
\textit{- NLOS Filtering by MGAL.}
After truncating the measured distance using the learned threshold, the resulting truncated measurement matrix $\hat{\mathbf{X}}$ still retains a portion of NLOS noise, as depicted by the gray region in Fig.~\ref{fig: ALMatt_VS_MearDist}. 
To examine the effectiveness of MGAL' attention mechanism in mitigating the remaining NLOS noise, we present the distributions of attention scores for both LOS and NLOS links in Fig.~\ref{fig: AttScore_VS_Noise}.
Primarily, it is evident that MGAL assign attention scores to NLOS and LOS links to distinct distributions, thus demonstrating the ability of the attention mechanism in MGAL to distinguish between NLOS and LOS links. Furthermore, the distribution of attention scores for NLOS links exhibits a smaller mean and variance compared to those of LOS links, indicating that NLOS links tend to receive lower attention scores, while LOS links exhibit relatively larger attention scores with a more flexible range of values. Consequently, in the feature propagation step (Eq.~\eqref{eq: dynamic_GAT_eqnatt}), MGAL assign smaller aggregation weights to NLOS links while prioritizing LOS links, which further filter out NLOS noise to some degree.

\vspace{4pt}
\noindent
\textit{- Validation of Dynamic Attention Property.}
{As detailed in Sec.~\ref{dynamic_att_property}, the dynamic attention property has been demonstrated for both the proposed ALM and MGAL methods. To validate this, we compute the attention scores between nodes with indices $\{1,\cdots,50\}$ using the well-trained parameters from GAT, ALM, and MGAL models. This process produces an attention score matrix denoted as $\bbE \in \mathbb{R}^{50 \times 50}$. For each row of the attention scores $\bbe_{[i,:]}$, we individually rank the scores, yielding three distinct patterns of attention score rankings corresponding to GAT, ALM, and MGAL, as illustrated in Fig.~\ref{fig:Attention_Comparison}.  
Specifically, Fig.~\ref{fig:Attention_Comparison}(a) depicts the results for the standard GAT. Here, the attention scores in each row maintain a consistent rank, reflecting a fixed order in assigning weights to the other nodes and indicating a static attention property.  	
In contrast, Fig.~\ref{fig:Attention_Comparison}(b) and (c) present the results for ALM and MGAL, respectively. In these modules, the attention scores exhibit varying ranks within each row, validating the dynamic attention property inherent to both ALM and MGAL.  
Additionally, a comparison of Fig.~\ref{fig:Attention_Comparison}(b) and (c) reveals that ALM’s attention scores exhibit symmetry, while those of MGAL do not. This distinction arises from the distance-aware attention mechanism employed in ALM, as defined in Eq.~\eqref{eq: neighbor_selection_att}, which imposes a symmetric structure on the learned attention scores.}

\section{Conclusion}
\label{sec: conclusion}
This paper investigated the robust cooperative localization of massive wireless networks in mixed LOS/NLOS environments. To address this challenging problem, we leveraged cutting-edge GNNs augmented with tailored attention mechanisms. Specifically, we proposed a novel AGNN model capable of autonomously learning the underlying graph structure of the network and the aggregation weight of each node. Consequently, our AGNN model outperforms our previous vanilla GCN-based method that employed predefined graph structures. Further theoretical analyses demonstrated that the proposed AGNN model exhibits dynamic attention properties and affordable computational complexity. Numerical results underscored the superiority of the class of GNN-based network localization methods in terms of localization accuracy, robustness, and computational efficiency, rendering them highly suitable for future large-scale networks in complex settings. The AGNN model, in particular, is explored by detailed dissection experiments to demonstrate significant advancements, establishing its potential for massive network applications in challenging environments.

\ifCLASSOPTIONcaptionsoff
\newpage
\fi

\bibliographystyle{IEEEtran}
\bibliography{Chapters/Ref.bib}
%

\onecolumn

\section{Appendix}

\subsection{The Optimization Method for $T_h^0$ Selection}

\label{sec: thre_learning_method}

Originally, ALM employs a manually predetermined threshold, denoted as $T_h^0$, uniformly applied across all nodes to establish a coarse-grained neighbor set. 
To address the issue related to the predetermined threshold, we introduce an optimization method capable of autonomously learning an individual threshold for each node. 
More precisely, the optimization method transforms the uniform threshold into a trainable vector (a scalar or a matrix case is analogous), with each element representing the threshold value for a specific node. This vector is randomly initialized and it can be updated and optimized through the backpropagation process. The specific approach is outlined as follows.

The trainable threshold vector is represented as $\bbt^v\in\mathbb{R}^{N}$. Since $\mathbf{t}^v$ serves as a trainable vector, it is imperative to ensure that each element within this vector is bounded within the interval $[0, l_\mathrm{max}]$, where $l_\mathrm{max}$ corresponds to the maximum distance over which a node in the network can receive signals. To enforce this constraint on $\mathbf{t}^v$ during the training process, we apply the following rescaling operation:
\begin{equation}
	\label{eq: threshold_transform}
	\hat{\bbt}^v =l_\mathrm{max}\cdot\mathrm{Sigmoid}(\bbt^v).
\end{equation}

Subsequently, it is natural to construct adjacency matrix $\mathbf{A}\in\mathbb{R}^{N \times N}$ as similar to Eq.~\eqref{eq:threshold} that
\begin{equation}\label{eq:thre_adj}
	a_{ij}=\begin{cases}
		0,\quad \mathrm{if} \quad  x_{ij}>\hat{t}^v_i, \\
		1,\quad \mathrm{otherwise},
	\end{cases}
\end{equation}
and further construct the truncated measurement matrix as $\hat{\mathbf{X}}=\mathbf{A} \odot \mathbf{X}$. 

Based on the approximated step function which is elucidated in Eq.~\eqref{eq: Approx_step_Fun},  Eq.~\eqref{eq:thre_adj} and relative expression of $\hat{x}_{ij}$ can be expressed as follows:
\begin{align}
	\label{eq: adj_transform}
	a_{ij} &= \mathrm{ReLU}(-\mathrm{tanh}(\gamma (x_{ij}-\hat{t}^v_i )))\\
	\label{eq: x_transform}
	\hat{x}_{ij} &= x_{ij}\cdot \mathrm{ReLU}(-\mathrm{tanh}(\gamma(x_{ij}-\hat{t}^v_i ))).
\end{align}
Once the learnable adjacency matrix and truncated measurement matrix are obtained via the optimization method, they can be effectively employed as inputs for MGAL. 

The optimization problem for this procedure can be formulated as follows:
\setlength{\jot}{1pt}
\begin{equation}
\begin{aligned}
	\arg\min_{\bbW, \bbt^v}~~~~~ \mathcal{L} &= \|\mathbf{P}_l-\hat{\mathbf{P}}_l \|^2_F\\[2\jot]
	\mathrm{s.t.} ~~~~~ \hat{\bbP} &= \text{MGAL}_{\bbW}(\bbA,\hat{\bbX})\\[\jot]
	a_{ij} &= \mathrm{ReLU}(-\mathrm{tanh}(\gamma(x_{ij}-\hat{t}^v_i )))\\[\jot]
	\hat{x}_{ij} &= x_{ij}\cdot \mathrm{ReLU}(-\mathrm{tanh}(\gamma(x_{ij}-\hat{t}^v_i )))\\[\jot]
	\hbt^v &=l_\mathrm{max}\cdot\mathrm{Sigmoid}(\bbt^v).
\end{aligned}
\end{equation}
Herein, $\bbW$ represents all trainable matrices in MGAL. 
Upon obtaining the optimized initial threshold vector, and its associated coarse-grained neighbor set, this set is employed in the \textit{fine-grained neighbor refinement} stage within the ALM. Subsequently, the application of MGAL ensues, ultimately leading to the attainment of the final localization.

\subsection{Proof of Theorem \ref{theorem: static_GAT}}
\label{sec: proof: static_GAT}

Let $\mathcal{G}=\left(\mathcal{V},\mathcal{E}\right)$ be a graph modeled by 
a GAT layer with 
some $\bbv$ and $\bbW$ values, and having node representations $\{\bbh_{[i,:]},..., \bbh_{[N,:]}\}$.
The learned parameter $\bbv$ can be written as a concatenation $\bbv=\left[\bbv_1 \| \bbv_2\right]\in\mathbb{R}^{2D_k}$ such that $ \bbv_1,\bbv_2\in\mathbb{R}^{D_k}$, and GAT can be re-written as:
\begin{equation}
	\alpha_{ij}= e\left(\bbh_{[i,:]}, \bbh_{[j,:]}\right)=
	\phi
	\left(
	\bbh_{[i,:]}\bbW \bbv_1+\bbh_{[j,:]}\bbW \bbv_2
	\right).
	\label{eq:gat-re}
\end{equation}
Since $\mathcal{V}$ is finite, 
there exists a node $j_{max}\in\mathcal{V}$ such that $\bbh_{[{j_{max},:]}}\bbW \bbv_2$ is maximal among all nodes $j\in\mathcal{V}$\polish{ ($j_{max}$ is the $j_f$ required by static)}{}.
Due to the monotonicity of $\phi(\cdot)$ and $\mathrm{softmax}$, for every query node $i \in\mathcal{V}$, the node $j_{max}$ also leads to the maximal value of its attention distribution $\{\alpha_{ij} \mid j \in \mathcal{V}\}$.
Thus, it computes only static attention.

\subsection{Proof of Theorem \ref{thm: dynamic_attention}}
\label{proof: dynamic_attention}

Let's consider a graph $\mathcal{G} = (\mathcal{V}, \mathcal{E})$. This graph comprises transformed node representations, denoted as $\{\bbh_{[i,:]}\bbW, \ldots, \bbh_{[N,:]}\bbW\}$. Suppose $\varphi: [N] \to [N]$ is any node mapping.

We introduce a function $g: \mathbb{R}^{2D_k} \to \mathbb{R}$, defined as follows:
\begin{equation}
    g\left(\bbx\right) = 
    \begin{cases}
        1 & \textrm{when }\bbx=\left[\bbh_{[i,:]}\bbW \|\bbh_{[\varphi\left(i\right),:]}\bbW\right] \\
        0 & \textrm{otherwise}
    \end{cases}
\end{equation}
Further, we define a continuous function $\widetilde{g}: \mathbb{R}^{2D_k} \to \mathbb{R}$, which matches $g$ at precisely $N^2$ specific inputs: $\widetilde{g}(\bbx)=g(\bbx)$, when $\bbx=\left[\bbh_{[i,:]}\bbW  \| \bbh_{[j,:]}\bbW \right], \forall i,j\in [N]$. For all other inputs $\bbx \in \mathbb{R}^{2D_k}$, $\widetilde{g}(\bbx)$ assumes any values that preserve its continuity.

Note that $\widetilde{g}$ is formulated for ease of proof. As our goal is for the graph attention layer's scoring function $e$ in Eq.~(\ref{eq: dynamic_GAT_att}, \ref{eq: feature_trans}) to approximate the mapping $\varphi$ at a finite set of points, we need $e$ to approximate $g$ at specific points. Although $g$ is discontinuous, $\widetilde{g}$ is continuous, facilitating the use of the universal approximation theorem. By approximating $\widetilde{g}$, $e$ effectively approximates $g$ at our specified points. We only require $\widetilde{g}$ to match $g$ at $N^2$ specific points ${ [\bbh_{[i,:]} | \bbh_{[j,:]}] : i, j \in [N]}$. Beyond these points, the $\widetilde{g}$'s values are flexible, as long as its continuity is maintained.

Therefore, for every node $i \in \mathcal{V}$ and every $j_{\neq \varphi(i)} \in \mathcal{V}$, we have:
\begin{equation}
    \underbrace{\widetilde{g}\left(\left[\bbh_{[i,:]} \bbW\|\bbh_{[j,:]} \bbW\right]\right)}_{=0} < \underbrace{\widetilde{g}\left(\left[\bbh_{[i,:]} \bbW\|\bbh_{[{\varphi\left(i\right)},:]}\bbW\right]\right)}_{=1}.
\end{equation}
If we concatenate the input vectors and define the graph attention layer's scoring function $e$ based on the concatenated vector $[\bbh_{[i,:]} \bbW | \bbh_{[j,:]} \bbW]$, then according to the universal approximation theorem, $e$ can approximate $\widetilde{g}$ for any compact subset of $\mathbb{R}^{2D_k}$.

Hence, for any sufficiently small $\epsilon$ (where $0 < \epsilon < \frac{1}{2}$), there exist parameters $\bbW_{att}$ and $\bbv_{att}$ such that for every node $i \in \mathcal{V}$ and every $j_{\neq \varphi(i)} \in \mathcal{V}$:
\begin{equation}
    \underbrace{e\left(\bbh_{[i,:]} \bbW,\bbh_{[j,:]} \bbW\right)}_{< 0 + \epsilon} < \underbrace{e\left(\bbh_{[i,:]} \bbW,\bbh_{[{\varphi\left(i\right)},:]}\bbW\right)}_{1 - \epsilon <}.
\end{equation}
Owing to the increasing monotonicity of the $\mathrm{softmax}$ function, this implies: 
\begin{equation}
    \alpha_{i,j} < \alpha_{i,\varphi\left(i\right)}.
\end{equation}
This shows that our proposed graph attention layer possesses dynamic attention, as it assigns the highest attention score to any neighbor node depending on the concatenated feature vector.

\subsection{Proof of Theorem \ref{thm: dynamic_attention_ALM-II}}
\label{proof: dynamic_attention_ALM-II}
The expression for the attention coefficient in ALM, as defined in Eq.~\eqref{eq: neighbor_selection_att}, can be reformulated as follows:
\begin{align}
	\label{eq: ALM-II_att_div}
	e_{ij}^A = \sum_{k=1}^{F_A} u_{ij}^k, \, \forall j \in \mathcal{N}_{i}^{C},
\end{align}
where $u_{ij}^k = \left| \phi \left( \bbx_{[i,:]}[\bbW_A]_k \right) -   \phi\left(\bbx_{[j,:]}[\bbW_A]_k \right) \right| \cdot[\bbv_A]_k$. Here, the notation $[\cdot]_k$ represents the $k$-th column vector or the $k$-th element of a matrix or vector.

Based on the decomposition of the summation in Eq.~\eqref{eq: ALM-II_att_div}, we focus on the $k$-th term, denoted as $u_{ij}^k$. Let $c_j^k:=\phi \left( \bbx_{[j,:]}[\bbW_A]_k\right)$, and then sort $c_j^k,\forall j\in \ccalN_i^C$ in ascending order to obtain the corresponding indices of the node sorting, denoted as $\left\lbrace j_{min},\cdots,j_{l_1}, j, j_{r_1},\cdots,j_{max}\right\rbrace$. Assuming the scalar $[\bbv_A]_k$ is positive, then for every $j\in\ccalN^C_i$, we have:
\begin{equation}
    u_{ij}^k \le \max\left\lbrace \left| c_i^k - c_{j_{max}}^k\right|, \left| c_i^k - c_{j_{min}}^k\right|\right\rbrace
\end{equation}
If the scalar $[\bbv_A]_k$ is negative, then for every $j\in\ccalN^C_i$, we have:
\begin{equation}
    u_{ij}^k \le \max\left\lbrace \left| c_i^k - c_{j_{l_1}}^k\right|, \left| c_i^k - c_{j_{r_1}}^k\right|\right\rbrace
\end{equation}

It is evident that the maximum value of $u_{ij}^k$ is not fixed but rather depends on $\ccalN^C$ and $[\bbv_A]_k$. Thus, $u_{ij}^k$, as the attention coefficient in the summation of $e_{ij}^A$ for the $k$-th term, exhibits a dynamic attention property.

Similar conclusions can be drawn for the remaining $F_A-1$ terms, wherein the difference arises due to the varying orders of $c_j^k,\forall j\in \ccalN_i^C$ determined by the specific values of $[\bbW_A]_k$ and $[\bbv_A]_k$, which further ensures the dynamic attention property.

\subsection{Proof of Theorem \ref{thm: gcn_complexity}}
\label{proof: gcn_complexity}
The complexity analysis of the $k$-th graph convolutional layer can be segmented into two components: Eq.~\eqref{eq:update_matrix} and Eq.~\eqref{eq:gcn_propagation}.

First, we consider the operation described in Eq.~\eqref{eq:update_matrix}, which entails a matrix product involving a sparse adjacency matrix and a dense representation matrix. Specifically, the $i$-th row of $\bar{\mathbf{H}}^{(k)}$ can be expressed as follows:
\begin{equation}
	\begin{aligned}
		\barbh_{[i,:]} = \hba_{[i,:]}\bbH^{(k-1)} 
		= \sum_{j=1}^{N} \hhata_{ij}\bbh_{[j,:]}^{(k-1)}
		= \sum_{j\in\ccalN_i} \hhata_{ij}\bbh_{[j,:]}^{(k-1)}.
	\end{aligned}
\end{equation}
Here, the computation of $\barbh_{[i,:]}$ demands $O(|\ccalN_i|D_{k-1})$ in terms of time complexity.
Given that $2|E| = \sum_{i=1}^N|\ccalN_i|$, the overall time complexity for Eq.~\eqref{eq:update_matrix} becomes $O(|E|D_{k-1})$.

Secondly, for the matrix multiplication between two dense matrices, as seen in Eq.~\eqref{eq:gcn_propagation}, the well-known time complexity is $O(ND_{k-1}D_k)$.

In summary, the time complexity of the $k$-th graph convolutional layer can be succinctly expressed as $O(ND_{k-1}D_k+|E|D_{k-1})$.

\subsection{Proof of Theorem \ref{thm: agnn_complexity}}
\label{proof: agnn_complexity}
The proof of Theorem \ref{thm: agnn_complexity} comprises two distinct parts, pertaining to the time complexity analysis of ALM and the $k$-th graph attention layer within the MGALs.

The primary computational complexity within ALM arises from the calculation of attention coefficients, as articulated in Eq.~\eqref{eq: neighbor_selection_att}. Firstly, we compute $\bbx_{[i,:]}\bbW_A$ for every $i\in[N]$, which necessitates $O(NNF_A)$. Subsequently, for each edge $(i,j)$, we compute $\left| \phi\left( \bbx_{[i,:]}\bbW_A \right) - \phi\left(\bbx_{[j,:]}\bbW_A \right) \right|$ using the pre-computed $\bbx_{[i,:]}\bbW_A$ and $\bbx_{[j,:]}\bbW_A$, incurring a time complexity of $O(|E^C|F_A)$. Finally, computing the results of the linear layer $\bbv_A$ adds an additional $O(|E^C|F_A)$. In summary, the time complexity of ALM amounts to $O(NNF_A+|E^C|F_A)$.

The computational complexity of the $k$-th graph attention layer within MGALs can be broken down into several components. Initially, we compute Eq.~\eqref{eq: feature_trans} for every $i\in[N]$, demanding $O(ND_{k-1}D_k)$. Subsequently, we calculate ${\hbh}_{[i,:]}^{(k-1)}\bbW_{att}^{(k-1)}$ for every $i\in[N]$ using the pre-obtained ${\hbh}_{[i,:]}^{(k-1)}$, which incurs a time complexity of $O(ND_kF_{att})$. Further, computing the results of the linear layer $\bbv_{att}^{{k-1}}$ takes $O(|E^F|F_A)$. Finally, we compute Eq.\eqref{eq: dynamic_GAT_eqnatt} for every $i\in[N]$, requiring $O(|E^F|D_k)$. Consequently, the time complexity for the $k$-th graph attention layer within MGALs is given by $O(ND_{k}F_{att}+ND_{k-1}D_k + |E^F|F_{att}+ |E^F|D_k)$.

By systematically analyzing each component, we have established the respective time complexities for ALM and the $k$-th graph attention layer within MGALs as outlined above.

\subsection{Comprehensive Analysis of CRB}
\label{app: analysis_CRB}

\subsubsection{The General Derivation of CRB}
We first provide a detailed derivation of the CR) under various NLOS conditions.
When establishing the performance bound, we assume precise knowledge of the actual measurement error distribution, denoted as $p_n(n)$. The entries corresponding to unknown positions of agents, represented by $\bbP_u:=[\bbp^x,\bbp^y] \in \mathbb{R}^{(N-N_l)\times 2}$, are reorganized into a vector form $\bbp_u^v:=[\bbp^x\|\bbp^y] \in \mathbb{R}^{2(N-N_l)}$ for ease of analysis.

According to \cite{Yin_ECM}, the Fisher information matrix (FIM) of $\bbp_u^v $ given the true $p_n(n)$ is readily obtained
\begin{equation}
\bbF\left(\bbp_u^v\right)=\left(\begin{array}{ll}
\bbF_{x x} & \bbF_{x y} \\
\bbF_{xy}^T & \bbF_{yy}
\end{array}\right)
\end{equation}
where $\bbF_{xx}, \bbF_{xy}$ and $\bbF_{yy}$ are all square matrices of dimension $(N-N_l) \times (N-N_l)$. The matrix elements are defined as:
\begin{equation}
\left[\bbF_{mn}\right]_{i, i^{\prime}}=\left\{\begin{array}{cl}
\mathcal{I}_n \cdot \sum_{\forall j \in \mathcal{N}(i)} \frac{\left(p^m_i-p^m_j\right)\left(p^n_i-p^n_j\right)}{\left\|[\bbP_u]_{[i,:]}-[\bbP_u]_{[j,:]}\right\|^2}, & i=i^{\prime} \\
-\mathcal{I}_n \cdot \delta_{i, i^{\prime}} \cdot \frac{\left(p^m_i-p^m_{i^{\prime}}\right)\left(p^n_i-p^n_{i^{\prime}}\right)}{\left\|[\bbP_u]_{[i,:]}-[\bbP_u]_{[i^\prime,:]}\right\|^2}, & i \neq i^{\prime}
\end{array},\right.
\end{equation}
for $m,n \in\{x, y\}$. Here, $\delta_{i, i^{\prime}}$ is Kronecker's delta defined by
\begin{equation}
\delta_{i, i^{\prime}}= \begin{cases}1, & \text { if } i^{\prime} \in \mathcal{N}(i) \\ 0, & \text { if } i^{\prime} \notin \mathcal{N}(i)\end{cases}.
\end{equation}
Additionally, the intrinsic accuracy scalar factor is given by:
\begin{equation}
	\mathcal{I}_{n}=\int \frac{\left[\nabla_{n} p_{n}\left(n\right)\right]^{2}}{p_{n}\left(n\right)} \mathrm{d} n.
\end{equation}
The scalar factor $\mathcal{I}_{n}$ is often approximated via Monte Carlo integration as
\begin{equation}
\label{eq: scalar_factor}
\mathcal{I}_{n} \approx \frac{1}{N_{M}} \sum_{n=1}^{N_{M}} \frac{\left[\nabla_{n} p_{n}\left(n^{\left(k\right)}\right)\right]^{2}}{p_{n}^{2}\left(n^{\left(k\right)}\right)}
\end{equation}
where $n^{(k)}, k=1,2, \ldots, N_M$ are i.i.d. samples generated from $p_n(n)$. 
Finally, the CRB is given by $\mathrm{CRB}\left(\bbp_u^v\right) := \bbF\left(\bbp_u^v\right)^{-1}$.

In the simulations, the localization accuracy is assessed in terms of the overall localization RMSE.
The lower bound of RMSE is given by
\begin{equation}
\overline{\mathrm{CRB}}\left(\bbp_u^v\right) := \sqrt{\frac{1}{N-N_l} \mathrm{tr}\left[\mathrm{CRB}\left(\bbp_u^v\right)\right]} .
\end{equation}

The subsequent focus is on deriving the PDF of noise $p_n$ and its derivative under distinct noise conditions. This will be detailed in Sections \ref{sec: Uniform_dist_NLOS} and \ref{sec: Rayleigh_dist_NLOS} for uniformly distributed NLOS and Rayleigh distributed NLOS, respectively.

\subsubsection{Uniformly Distributed NLOS Case}
\label{sec: Uniform_dist_NLOS}

For scenarios where the noise is a combination of additive Gaussian distribution (LOS) and Uniform distribution (NLOS), with an occurrence probability of $p_B$, the noise expression is given by:
\begin{equation}
\label{eq: LOS_&_NLOS}
    n =
    \begin{cases}
        n^L + n^N, & \text{w.p. } p_B, \\
        n^L, & \text{w.p. } 1 - p_B,
    \end{cases}
\end{equation}
where $n^L$ and $n^N$ are random variables with Gaussian distribution and Uniform distribution, respectively. The PDFs of $n^L$ and $n^N$  are shown as follows.

\begin{equation}
	p_{n^L} \left(x\right)= \frac{1}{\sqrt{2\pi}\sigma}e^{-\frac{x^2}{2\sigma^2}}, \quad -\infty<x<\infty,
\end{equation}

\begin{equation}
	p_{n^N} \left(x\right)=\frac{1}{b-a}, \quad a<x<b.
\end{equation}

To derive the PDF of noise $n$, we first need to derive the PDF of $n_1=n^L+n^N$. The probability distribution of the sum of two or more independent random variables is the convolution of their distributions. Thus, we have
\begin{subequations}
\begin{align}
	p_{n_1}\left(x\right)&=\int_{-\infty}^{\infty}p_{n^N} \left(u\right)p_{n^L} \left(x-u\right)du\\
	&=\frac{1}{b-a}\int_{a}^{b}\frac{1}{\sqrt{2\pi}\sigma}e^{-\frac{\left(x-u\right)^2}{2\sigma^2}}du\\
	&=\frac{1}{b-a}\frac{1}{2\sqrt{2}\sigma}\frac{2}{\sqrt{\pi}}\int_{a}^{b}e^{-\frac{\left(x-u\right)^2}{2\sigma^2}}du\qquad \left(\text{Let } t=\frac{x-u}{\sqrt{2}\sigma}, \text{then }u=x-\sqrt{2}\sigma t\right)\\
	&=\frac{1}{b-a}\frac{1}{2\sqrt{2}\sigma}\frac{2}{\sqrt{\pi}}\int_{a}^{b}e^{-t^2}d\left(x-\sqrt{2}\sigma t\right)\\
	&=\frac{1}{b-a}\frac{1}{2\sqrt{2}\sigma}\frac{2}{\sqrt{\pi}}\left(-\sqrt{2}\sigma\right)\int_{\frac{x-a}{\sqrt{2}\sigma}}^{\frac{x-b}{\sqrt{2}\sigma}}e^{-t^2}dt\\
	&=-\frac{1}{2\left(b-a\right)}\frac{2}{\sqrt{\pi}}\left(\int^{\frac{x-b}{\sqrt{2}\sigma}}_{0}e^{-t^2}dt-\int^{\frac{x-a}{\sqrt{2}\sigma}}_{0}e^{-t^2}dt\right) \qquad \left(\text{We know }erf\left(x\right)=\frac{2}{\sqrt{\pi}}\int_{0}^{x}e^{-t^2}dt\right)\\
	&=-\frac{1}{2\left(b-a\right)}\left(erf\left(\frac{x-b}{\sqrt{2}\sigma}\right)-erf\left(\frac{x-a}{\sqrt{2}\sigma}\right)\right)\\
	&=\frac{1}{2\left(b-a\right)}\left(erf\left(\frac{b-x}{\sqrt{2}\sigma}\right)-erf\left(\frac{a-x}{\sqrt{2}\sigma}\right)\right).
\end{align}
\end{subequations}
Now, since we know the PDF of $n_1$, the PDF of $n$ can be derived as:
\begin{subequations}
\begin{align}
		p_{n}\left(x\right)&=	p_{n_1}\left(x\right)p_B+	p_{n^L}\left(x\right)\left(1-p_B\right)\\
		&=\frac{p_B}{2\left(b-a\right)}\left(erf\left(\frac{b-x}{\sqrt{2}\sigma}\right)-erf\left(\frac{a-x}{\sqrt{2}\sigma}\right)\right)+\frac{1-p_B}{\sqrt{2\pi}\sigma}e^{-\frac{x^2}{2\sigma^2}}.
\end{align}
\end{subequations}
Then, the derivation of $p_{n}\left(x\right)$ is given by:
\begin{equation}
	\nabla_x p_{n}\left(x\right)=\frac{p_B}{2\left(b-a\right)}\left(-\frac{1}{\sqrt{2}\sigma}\right)\left(\frac{2}{\sqrt{\pi}}\right)\left(e^{-\left(\frac{b-x}{\sqrt{2}\sigma}\right)^2}-e^{-\left(\frac{a-x}{\sqrt{2}\sigma}\right)^2}\right)+\frac{1-p_B}{\sqrt{2\pi}\sigma}e^{-\frac{x^2}{2\sigma^2}}\left(-\frac{x}{\sigma^2}\right).
\end{equation}

\subsubsection{Rayleigh Distributed NLOS Case}
\label{sec: Rayleigh_dist_NLOS}

In the case where noise consists of additive Gaussian distribution (LOS) and Rayleigh distribution (NLOS) with occurrence probability $p_B$, the PDFs of $n^L$ and $n^N$ are shown as follows.

\begin{equation}
	p_{n^L} \left(x\right)= \frac{1}{\sqrt{2\pi}\sigma_1}e^{-\frac{x^2}{2\sigma_1^2}}, \quad -\infty<x<\infty,
\end{equation}

\begin{equation}
	p_{n^N} \left(x\right)=\frac{x}{\sigma_2^2}e^{-\frac{x^2}{2\sigma_2^2}}, \quad 0<x<\infty.
\end{equation}

The PDF of the combined noise $n_1=n^L+n^N$ is derived through convolution, yielding:
\begin{equation}
	p_{n_1} \left(x\right)=\frac{\sigma_2x}{\left(\sigma_1^2+\sigma_2^2\right)}e^{-\frac{x^2}{2\left(\sigma_1^2+\sigma_2^2\right)}}\phi\left(\frac{\sigma_2x}{\sigma_1\sqrt{\left(\sigma_1^2+\sigma_2^2\right)}}\right)+
	\frac{\sigma_1}{\sqrt{2\pi}\left(\sigma_1^2+\sigma_2^2\right)}e^{-\frac{x^2}{2\sigma_1^2}},
\end{equation}
where $\phi\left(x\right)$ is the cumulative distribution function of the standard normal random variable. Thus the PDF of $n$ is then expressed as:
\begin{subequations}
\begin{align}
	p_{v}\left(x\right)&=	p_{n_1}\left(x\right)p_B+	p_{n^L}\left(x\right)\left(1-p_B\right)\\ &=p_B\left(\frac{\sigma_2x}{\left(\sigma_1^2+\sigma_2^2\right)}e^{-\frac{x^2}{2\left(\sigma_1^2+\sigma_2^2\right)}}\phi\left(\frac{\sigma_2x}{\sigma_1\sqrt{\left(\sigma_1^2+\sigma_2^2\right)}}\right)+
	\frac{\sigma_1}{\sqrt{2\pi}\left(\sigma_1^2+\sigma_2^2\right)}e^{-\frac{x^2}{2\sigma_1^2}}\right)
	+\left(1-p_B\right)\left(\frac{1}{\sqrt{2\pi}\sigma}e^{-\frac{x^2}{2\sigma^2}}\right).
\end{align}
\end{subequations}
Then, the derivation of $p_{v}\left(x\right)$ is

\begin{equation}
\begin{aligned}
    \nabla_x p_{v}(x) = & \, p_B \left[ \frac{\sigma_2}{\sigma_1^2 + \sigma_2^2} \, e^{-\frac{x^2}{2(\sigma_1^2 + \sigma_2^2)}} \phi \left( \frac{\sigma_2 x}{\sigma_1 \sqrt{\sigma_1^2 + \sigma_2^2}} \right) + \frac{\sigma_2 x}{\sigma_1^2 + \sigma_2^2} \, e^{-\frac{x^2}{2(\sigma_1^2 + \sigma_2^2)}} \left( -\frac{x}{\sigma_1^2 + \sigma_2^2} \right) \phi \left( \frac{\sigma_2 x}{\sigma_1 \sqrt{\sigma_1^2 + \sigma_2^2}} \right) \right.\\
    & \quad + \left. \frac{\sigma_2 x}{\sigma_1^2 + \sigma_2^2} \, e^{-\frac{x^2}{2(\sigma_1^2 + \sigma_2^2)}} \frac{e^{-\left( \frac{\sigma_2 x}{\sigma_1 \sqrt{\sigma_1^2 + \sigma_2^2}} \right)^2 / 2}}{\sqrt{2\pi}} \frac{\sigma_2}{\sigma_1 \sqrt{\sigma_1^2 + \sigma_2^2}} + \frac{\sigma_1}{\sqrt{2\pi}(\sigma_1^2 + \sigma_2^2)} \, e^{-\frac{x^2}{2\sigma_1^2}} \left( -\frac{x}{\sigma_1^2} \right) \right] \\
    & + (1 - p_B) \left[ \frac{1}{\sqrt{2\pi}\sigma} e^{-\frac{x^2}{2\sigma^2}} \left( -\frac{x}{\sigma^2} \right) \right].
\end{aligned}
\end{equation}

\subsubsection{Experimental Results for Rayleigh Distributed NLOS}

In this subsection, we explore the scenario where NLOS conditions follow a Rayleigh distribution denoted as $\mathcal{R}(\sigma_2)$. 
The RMSE results of CRB and GNN-based methods (including GCN \cite{yan2021graph} and newly proposed AGNN) are shown in Tab.~\ref{tab:4}. The results indicate that, under Rayleigh-distributed NLOS conditions, AGNN consistently outperforms GNN in terms of localization accuracy. Furthermore, AGNN demonstrates proximity to the CRB across various levels of Rayleigh noise, highlighting its effectiveness and robustness in diverse noise environments.

\begin{table}[h]
	\centering
        \fontsize{10}{11}\selectfont
        \renewcommand{\arraystretch}{1.5}
	\caption{The averaged loss (RMSE) of CRB and GNN-based methods under different noise conditions for $N_l$=50.}
	\label{tab:4}       
	\begin{tabular}{c|ccccc}
		\hline
		methods$\backslash$ Noise $\left(\sigma_1, \sigma_2,p_B\right)$ & $\left(0.1,0.5,10\%\right)$ & $\left(0.1,1,10\%\right)$ & $\left(0.25,1,30\%\right)$ & $\left(0.25,3,30\%\right)$ & $\left(0.25,5,50\%\right)$  \\
            \hline
		GCN  & 0.1085 & 0.1111 & 0.1126 & 0.1319 & 0.1599\\
            AGNN & \textbf{0.0514} & \textbf{0.0561} & \textbf{0.0764} & \textbf{0.0805} &\textbf{0.1031}\\
		\hline
            CRB  & 0.0351 & 0.0423 & 0.0671 & {0.0714} &{0.0986}\\
            \hline
	\end{tabular}
\end{table}

\subsection{The Approximated Step Function}
\label{app: appro_step_func}
As shown in Fig.~\ref{fig: appro_step_func}, the approximated step function, employing ReLU and tanh, yields a notably sharper transition on one side and effectively truncates values to zero on the other side, consequently achieving a more accurate approximation of the step function than the sigmoid function. Moreover, its derivative can be obtained everywhere (In the standard definition of ReLU, the derivative at 0 is typically taken to be 0.), which makes it feasible to compute gradients with respect to the trainable threshold.

\setcounter{figure}{11} 
\begin{figure}[h]
	\centering
	\includegraphics[width=0.6\linewidth]{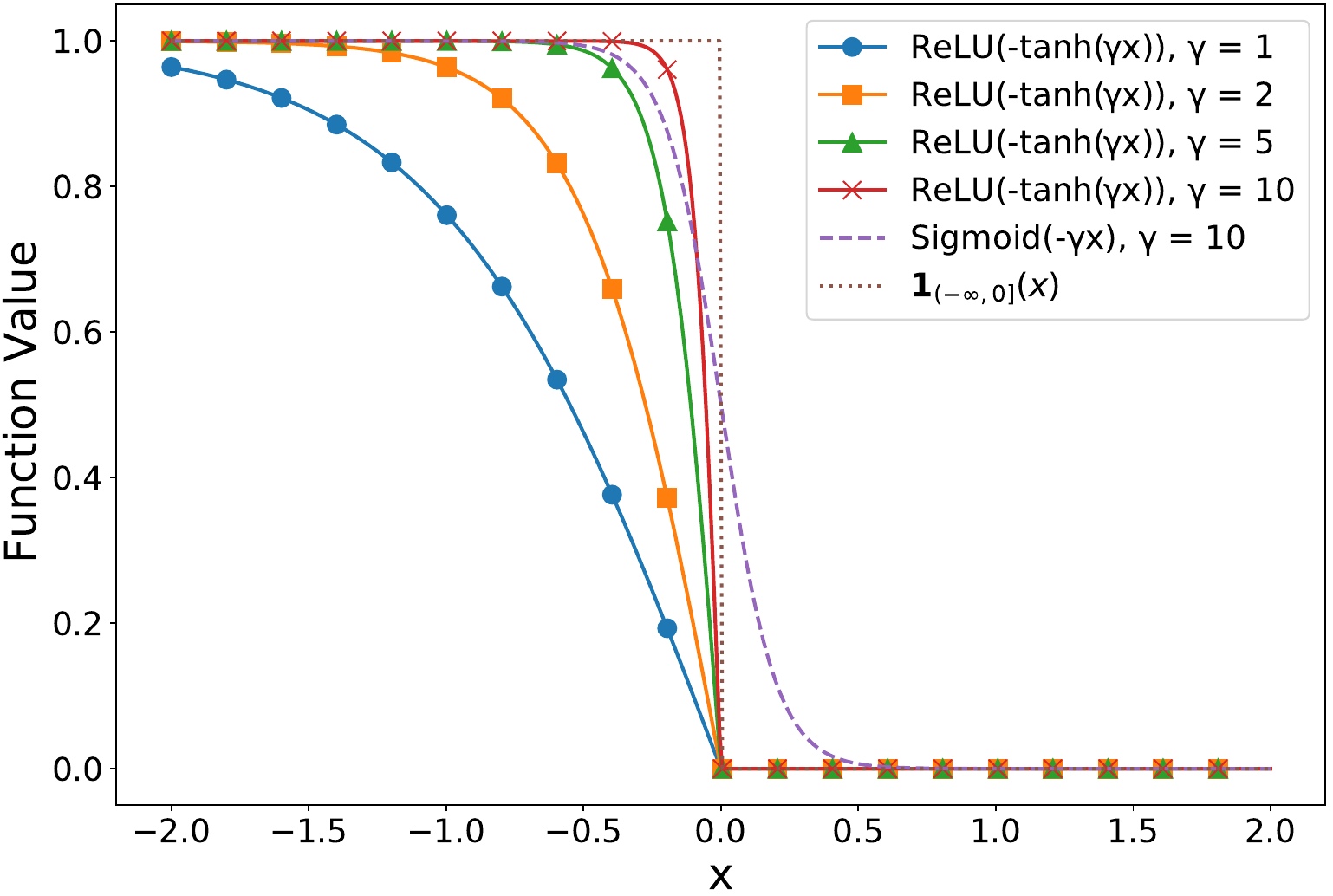}
	\caption{Comparison between sigmoid function, step function, and approximated step functions.}
	\label{fig: appro_step_func}
\end{figure}

\subsection{Additional Experimental Results}
\label{app: additional_experiments}

To further illustrate the generality of our experimental findings,  we present additional scatter plots depicting the relationship between true distances and learned thresholds for the $100$-th and the $200$-th nodes, as depicted in Fig.~\ref{fig: ALMatt_VS_MearDist_100th} and Fig.~\ref{fig: ALMatt_VS_MearDist_200th}, respectively.
Furthermore, scatter plots (Fig.~\ref{fig: adj_AdjAGNNII_Noise_100th} and Fig.~\ref{fig: adj_AdjAGNNII_Noise_200th}) are provided to demonstrate the coarse-grained neighbors ($\ccalN_{100}^C$ and $\ccalN_{200}^C$), fine-grained neighbors ($\ccalN_{100}^F$ and $\ccalN_{200}^F$), and the noise ($\bbn_{100}$ and $\bbn_{200}$) associated with the $100$-th and the $200$-th nodes.

\begin{figure}[h] 
	\centering
	\includegraphics[width=0.6\linewidth]{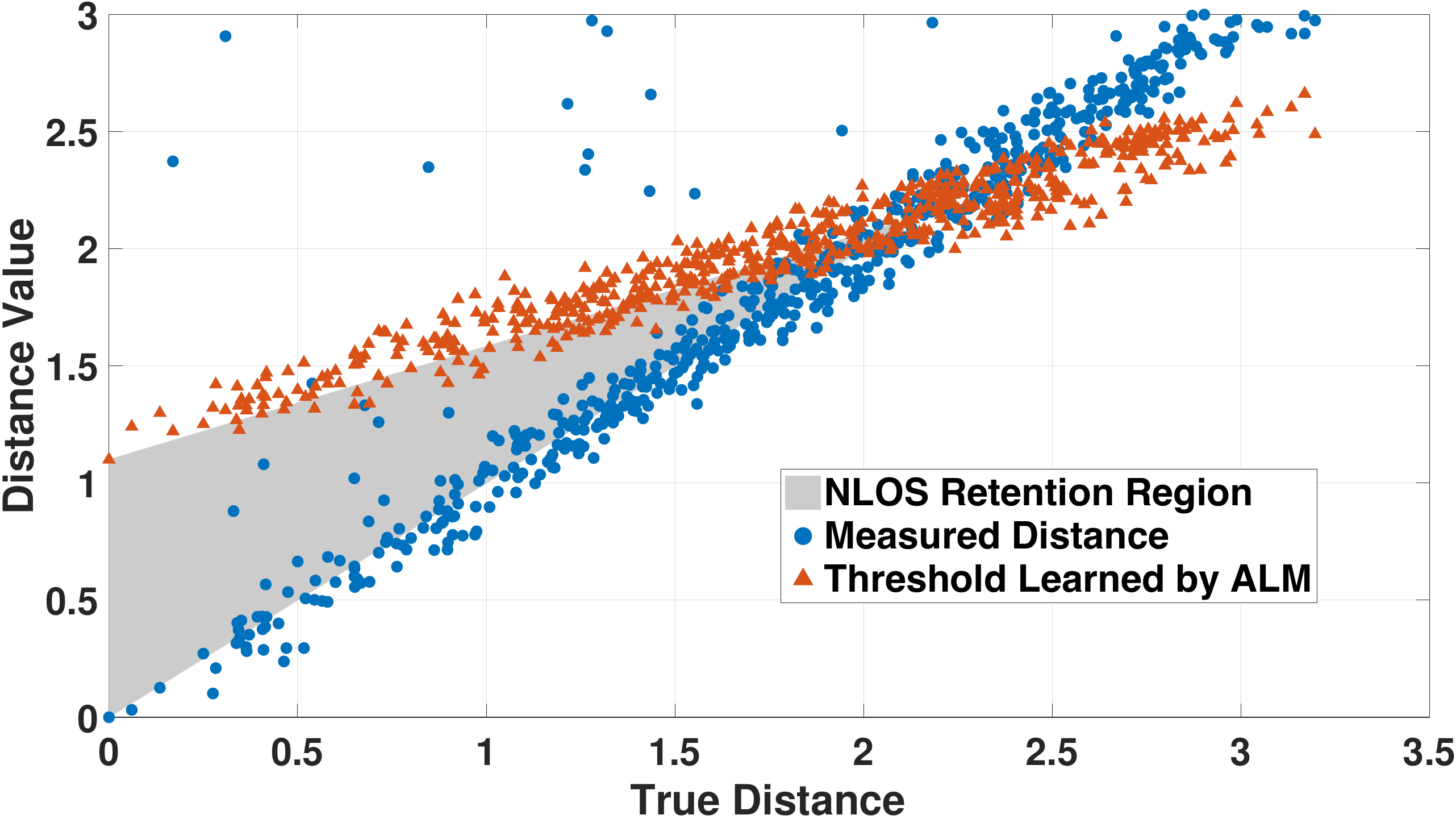}
	\caption{The threshold learned by ALM versus the measured distance for the $100$-th node in $(0.25,30\%)$ dataset with $T_h^0=3.0$.}
	\label{fig: ALMatt_VS_MearDist_100th}
\end{figure}

\begin{figure}[h] 
	\centering
	\includegraphics[width=0.6\linewidth]{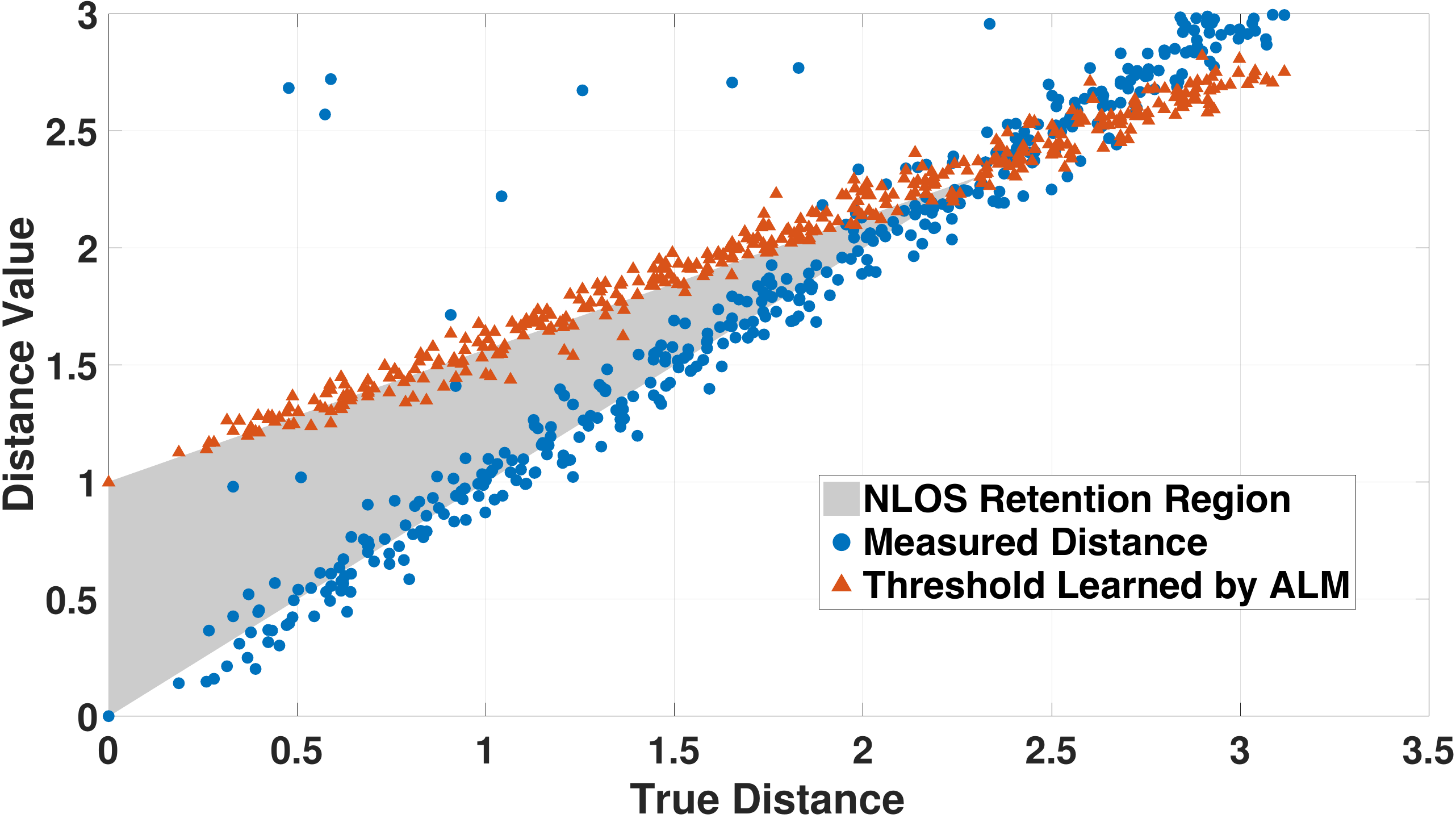}
	\caption{The threshold learned by ALM versus the measured distance for the $200$-th node in $(0.25,30\%)$ dataset with $T_h^0=3.0$.}
	\label{fig: ALMatt_VS_MearDist_200th}
\end{figure}

\begin{figure*}[h] 
	\centering
	\includegraphics[width=0.95\linewidth]{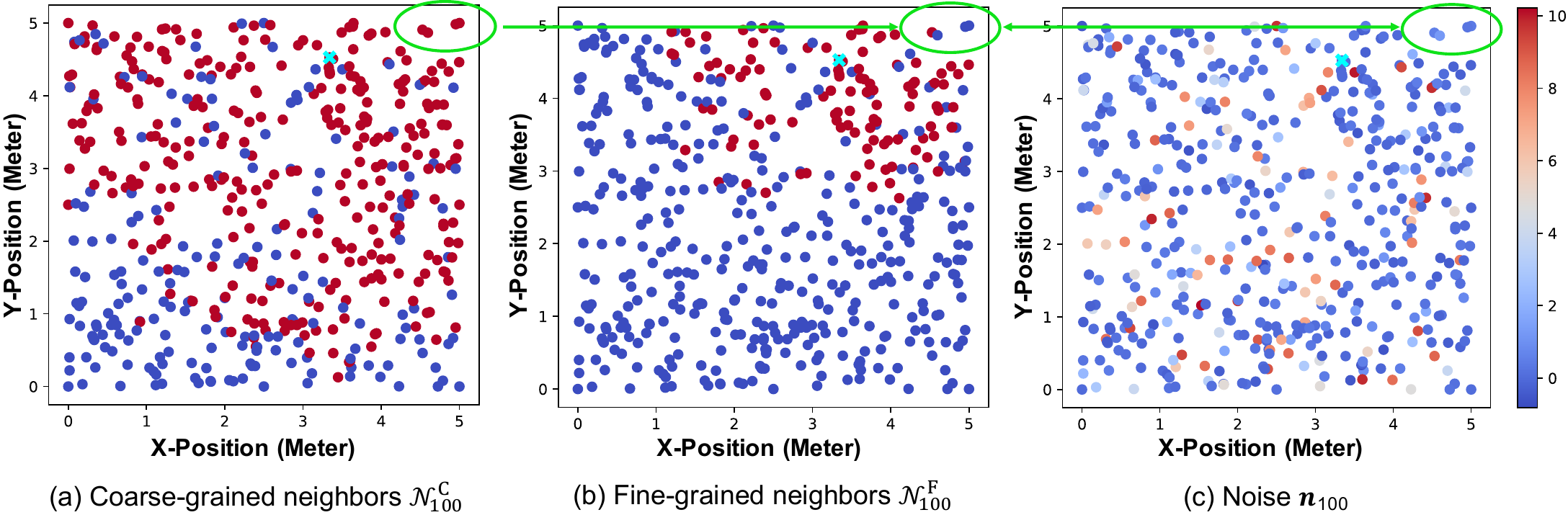}
	\caption{The different types of scatter plots for the $100$-th node, denoted by a cyan cross, in $(0.25,30\%)$ dataset with $T_h^0=4.0$.}
	\label{fig: adj_AdjAGNNII_Noise_100th}
\end{figure*}

\begin{figure*}[h] 
	\centering
	\includegraphics[width=0.95\linewidth]{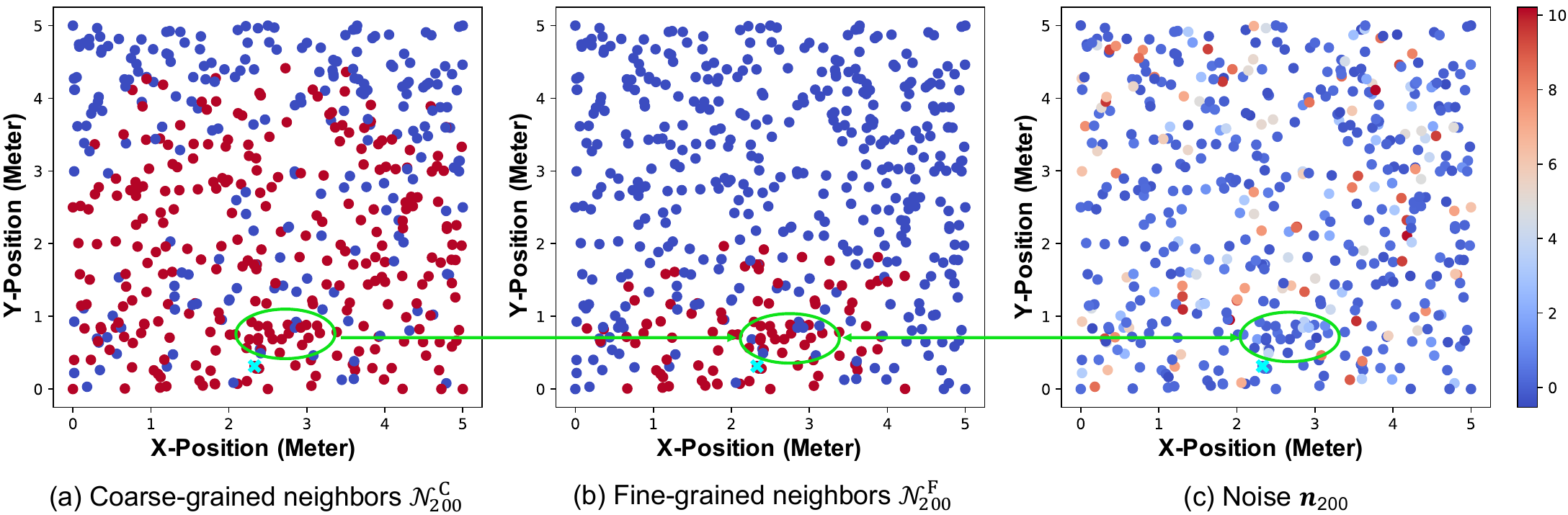}
	\caption{The different types of scatter plots for the $200$-th node, denoted by a cyan cross, in $(0.25,30\%)$ dataset with $T_h^0=4.0$.}
	\label{fig: adj_AdjAGNNII_Noise_200th}
\end{figure*}

\end{document}